\documentclass[sigconf]{acmart}

\usepackage{url}            
\usepackage{booktabs}       
\usepackage{amsfonts}       
\usepackage{nicefrac}       
\usepackage{microtype}      
\usepackage{xcolor}         
\usepackage{tcolorbox}
\usepackage{amsmath}
\usepackage{graphicx}
\usepackage{subfigure}
\usepackage{subcaption}
\usepackage{enumerate}
\usepackage{amsmath}
\usepackage{mathtools}
\usepackage{amsthm}
\usepackage[table]{xcolor}
\usepackage{algorithm}
\usepackage{algorithmic}
\usepackage{wrapfig}
\usepackage{caption}
\usepackage[normalem]{ulem}
\usepackage{multirow}
\usepackage{array}
\usepackage{multicol}
\usepackage{tikz}
\usepackage{makecell}
\theoremstyle{plain}
\newtheorem{theorem}{Theorem}[section]

\newtheorem{corollary}[theorem]{Corollary}
\theoremstyle{definition}
\newtheorem{definition}[theorem]{Definition}
\newtheorem{assumption}[theorem]{Assumption}
\theoremstyle{remark}

\definecolor{myred}{RGB}{243,125,119}
\definecolor{deepdeepred}{RGB}{107,0,12}
\definecolor{mydeepred}{RGB}{178,23,41}
\definecolor{mygreen}{RGB}{112,195,121}
\definecolor{myblue}{RGB}{0,148,254}
\definecolor{mypurple}{RGB}{128,128,254}
\definecolor{mydeepblue}{RGB}{117,112,179}

\definecolor{thered}{RGB}{247,105,108}   
\definecolor{thegreen}{RGB}{99,190,123}  
\definecolor{myred}{HTML}{D3393C}
\definecolor{mygreen}{HTML}{3A9D5D}

\AtBeginDocument{%
  }

\setcopyright{acmlicensed}
\copyrightyear{2018}
\acmYear{2018}
\acmDOI{XXXXX.XXXXX}
\acmConference[Conference acronym 'XX]{Make sure to enter the correct
  conference title from your rights confirmation email}{June 03--05,
  2018}{Woodstock, NY}

\acmISBN{978-1-4503-XXXX-X/2018/06}

\def\method{Nomad}

\begin{document}

\title{Relative Transitions, Not Absolute Destinations: A Transfer-and-Ground Framework for Target-Trajectory-Free Human Mobility Generation}

\author{Yidi Wang}
\email{wang.yidi@connect.um.edu.mo}
\affiliation{%
  \institution{SKL-IOTSC, Department of CIS, University of Macau}
  \state{Macau SAR}
  \country{China}
}

\author{Yunhe Zhang}
\affiliation{%
  \institution{SKL-IOTSC, Department of CIS, University of Macau}
  \state{Macau SAR}
  \country{China}}
\email{zhangyhannie@gmail.com}

\author{Bangchao Deng}
\affiliation{%
  \institution{SKL-IOTSC, Department of CIS, University of Macau}
  \state{Macau SAR}
  \country{China}}
\email{yc37980@um.edu.mo}

\author{Dingqi Yang}
\affiliation{%
  \institution{SKL-IOTSC, Department of CIS, University of Macau}
  \state{Macau SAR}
  \country{China}}
\email{dingqiyang@um.edu.mo}

\author{Pengyang Wang}
\correspondingauthor
\affiliation{%
  \institution{SKL-IOTSC, Department of CIS, University of Macau}
  \state{Macau SAR}
  \country{China}}
\email{pywang@um.edu.mo}

\renewcommand{\shortauthors}{Trovato et al.}

\begin{abstract}

Individual mobility trajectories support urban analysis and location-based services, yet most trajectory generators require observations from their deployment city.
This assumption excludes precisely the cities where trajectories are unavailable even though points of interest (POIs) and their attributes can be obtained from public maps.
We study \emph{target-trajectory-free generation}: learning from POIs and trajectories in source cities while utilizing only POI coordinates and categories in a target city, with no target trajectory or trajectory-derived statistic available for training, model selection, or generation.
Existing trajectory generators typically predict absolute destinations, entangling reusable movement behavior with city-specific POI identities and spatial layouts.
Our core insight is to replace this city-bound output with context-conditioned relative transitions.
We propose \textit{\method{}}, a transfer-and-ground framework that separates learning how people move from determining where those movements are realized.
Specifically, a history-conditioned flow-matching model learns from source trajectories a transition prior over semantic displacement between POI contexts, geographic displacement, and elapsed time; at inference, a behavior graph and an exploration--return walk ground sampled transitions onto the target POI map.
This factorization enables a direct test of representation level transferability without assuming invariance of the full mobility distribution.
Extensive experiments across ten cities and 14 transfers show that \textit{\method{}} outperforms adaptation baselines in trajectory fidelity and downstream utility,  lowering the average error over the best baseline of each metric by about 15\% in distributional fidelity and about 3\% in downstream utility.
Mechanism-level analyses further show that, under matched contexts, relative transitions contain less city-specific information and transition generation suffers less cross-city degradation than its absolute destination counterpart.
Together, these findings suggest that separating movement relations from city-specific place identities provides a more transferable basis for mobility synthesis in unseen cities without target-city trajectories.

\end{abstract}

\ccsdesc[500]{Computing methodologies~Modeling and simulation}
\received{20 February 2007}
\received[revised]{12 March 2009}
\received[accepted]{5 June 2009}

\keywords{Human Mobility; Transfer Learning; Simulation; Flow Matching}

\maketitle

\section{Introduction}

Individual mobility trajectories support transportation and travel-demand analysis \cite{wang2019urban}, epidemic simulation \cite{eubank2004modelling,chang2021mobility}, location-based services \cite{cho2011friendship,yang2014modeling,yin2014temporal}, and the study of urban function \cite{yuan2012discovering,liu2022developing}.
However, most trajectory generators assume that the deployment city already provides enough trajectories for training or adaptation.
A newly deployed location-based service can only obtain the city's map before it has accumulated local check-in histories.
Alternatively, target trajectories may be held by another platform or institution but remain inaccessible because of data privacy concerns~\cite{bindschaedler2016synthesizing}.
By contrast, public mapping platforms provide POI coordinates and categories without exposing individual movements \cite{haklay2008openstreetmap}.
This asymmetry motivates a different question from conventional in-city generation: \emph{can a source-trained model generate target-city trajectories given only the target POI map at deployment?}

We formalize these deployment conditions as a strict setting, namely \emph{target-trajectory-free}.
Here, \emph{target-trajectory-free} does not mean that no trajectories exist in the city, but that none are available to the model developer.
A model may leverage POIs and trajectories from one or more source cities.
For the target city, it is restricted to POI-relevant information such as coordinates and categories, but it may not access any target trajectories or trajectory-derived statistics.
The benchmark reserves target trajectories exclusively for final diagnosis and evaluation.
This information boundary distinguishes the problem from domain adaptation and few-shot transfer: the target city is visible as a set of places, but its mobility behavior is unavailable.

Existing methods do not naturally satisfy this boundary because they generate a city-specific object. In-city adversarial, imitation, and diffusion models learn directly from trajectories in the deployment city \cite{feng2020learning,choi2021trajgail,rao2020lstm,zhu2023difftraj,zhu2024controltraj,deng2025revisiting}.
Cross-city and meta-learning methods reduce the amount of target supervision, but typically retain target observations, a shared spatial grid, city-specific location embeddings, or a target output layer \cite{pan2019urban,jin2022selective,wang2018cross,wang2024cola,wang2025gtg}.
Their common prediction target is the next absolute destination \(p_k\), whose identity and empirical frequency are defined only within one city's POI vocabulary.
When target trajectories are absent, the problem is not merely that this distribution is estimated from fewer samples: the target-specific output probabilities have never been observed.

\begin{figure}[t]
\vskip -0.2in
\begin{center}
\centerline{\includegraphics[width=\linewidth]{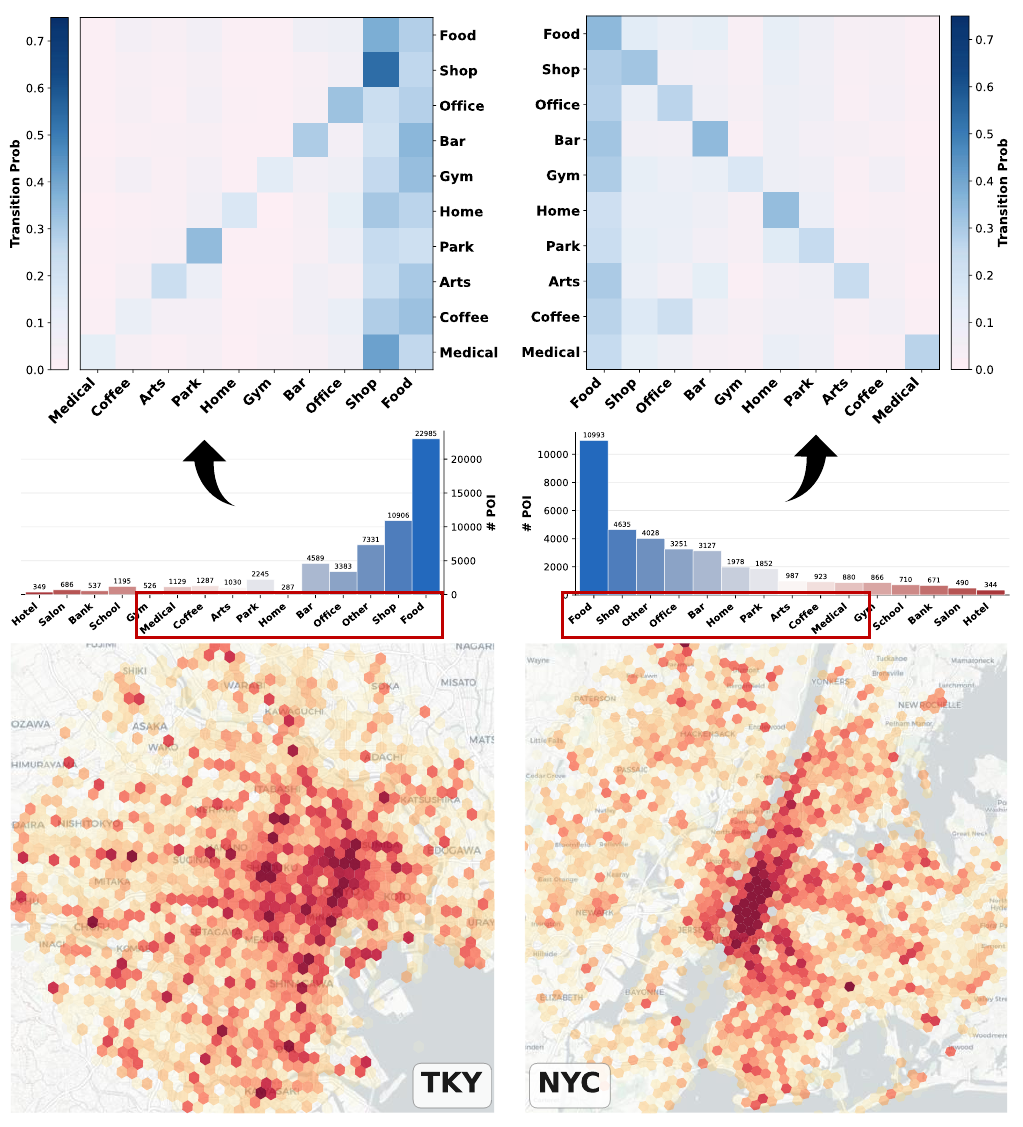}}
\vspace{-5pt}
\caption{NYC vs. TKY: same transitions, different POIs.
}
\label{fig:motivation}
\end{center}
\vskip -0.2in
\end{figure}

Human mobility nevertheless exhibits recurring statistics across cities, including heavy-tailed travel distances, exploration--return patterns, and repeated activity motifs \cite{brockmann2006scaling,gonzalez2008understanding,song2010modelling,pappalardo2015returners,schneider2013unravelling}.
These observations motivate transfer, but they do not establish that complete mobility distributions are invariant across cities.
Fig. \ref{fig:motivation} compares absolute destinations and relative transitions on NewYork (NYC) and Tokyo (TKY). The spatial layout and the popularity of POIs differ between the two cities. In contrast, the transition distribution is highly consistent across the two cities at the level of the functional semantics of places. It demonstrates that relative transitions carry less city-specific information than absolute destinations.

Motivated by this hypothesized difference, we reformulate cross-city generation around relative transitions rather than absolute destinations.
We transfer a source-learned prior over context conditioned relative transitions to capture \emph{how people move}, and then ground its samples on target-city POIs to determine \emph{where those movements are realized}.
Unlike an absolute destination, such a transition is defined by relations between place contexts rather than by a city-private POI identity.
We realize this separation with \textit{\method{}}, a transfer-and-ground framework.
Because POIs in different cities share neither identities nor comparable coordinates, a multi-relational graph encoder represents each POI through its category, coordinates, and spatial-semantic neighborhood, while a cross-city objective maps functionally similar contexts into a comparable representation space.
A history-conditioned flow-matching model then learns from source trajectories a prior over semantic displacement, geographic displacement, and elapsed time in this space.
Modeling relations rather than POI identities allows the prior to incorporate temporal and trajectory history without treating source-specific locations as prediction targets.
The sampled transitions are continuous, whereas a target trajectory consists of discrete POIs.
\textit{\method{}} therefore snaps each sample to a target POI and aggregates these assignments into a behavior graph without accessing target trajectories.
Walking this graph accounts for exploration, but a relative displacement cannot specify a return to a particular previously visited place.
An exploration--return pointer walk handles this distinction by either following a graph edge or revisiting a POI from the generated history.
Thus, source trajectories determine how people move, while the target map determines where those movements can be realized.
The main contributions are:
\begin{enumerate}[0]
\item[$\bullet$] We formulate target-trajectory-free mobility generation and specify an evaluation protocol in which the target contributes a POI map but no trajectory-derived training or generation signal.
\item[$\bullet$] We introduce a transfer-and-ground factorization that represents movement through context-conditioned semantic displacement, geographic displacement, and elapsed time, then realizes sampled transitions on the target POI map through behavior-graph grounding and exploration--return modeling.
\item[$\bullet$] Across 10 cities and 14 source-to-target transfers, \method{} improves trajectory fidelity by about 15\% and downstream train-on-synthetic--test-on-real utility by about 3\% over six baselines. Controlled analyses further show that relative transitions contain less city-specific information under matched contexts and incur less cross-city degradation than absolute-destination generation, directly supporting the proposed mechanism.
\end{enumerate}

\begin{figure*}[ht]
\vskip -0.2in
\begin{center}
\centerline{\includegraphics[width=\linewidth]{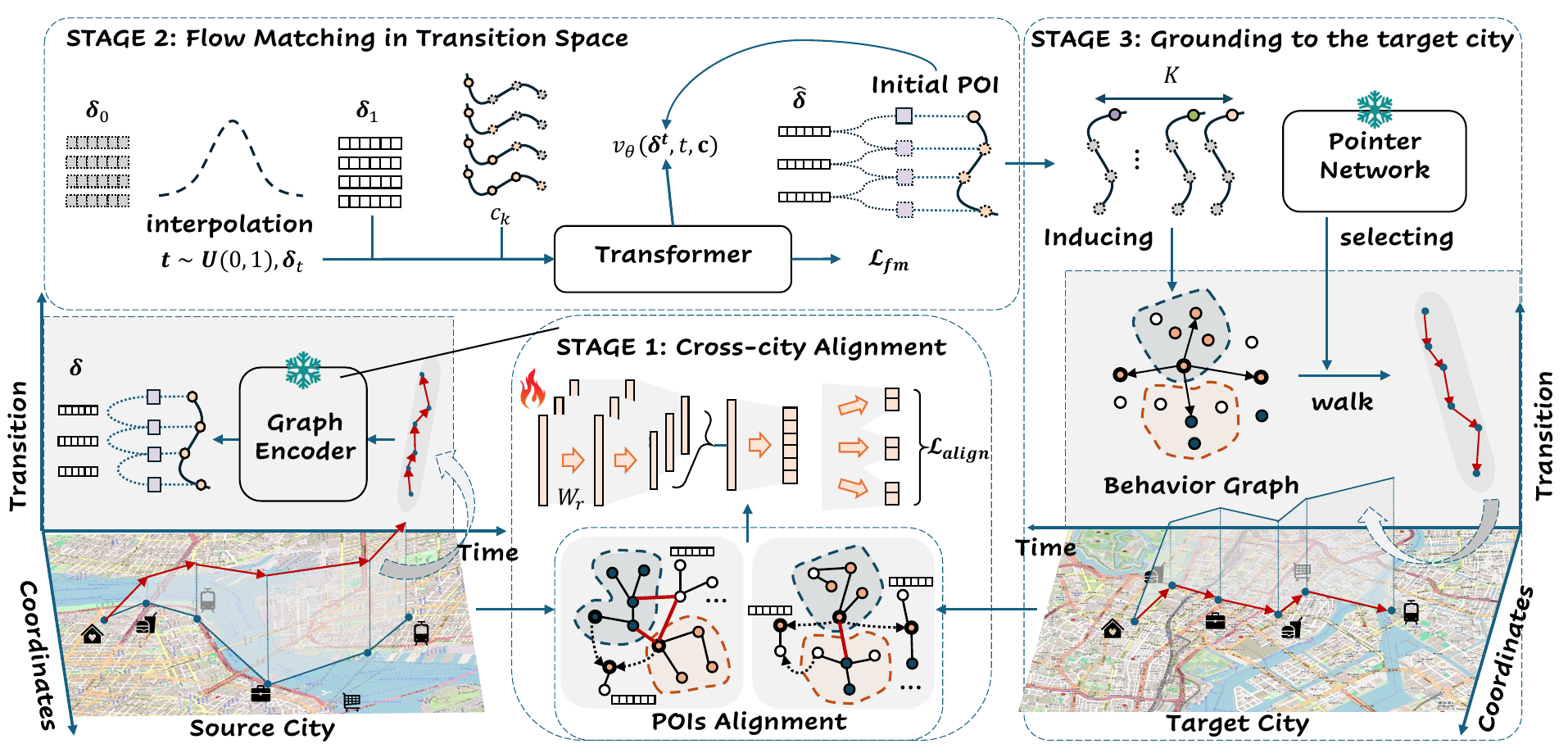}}
\vspace{-5pt}
\caption{Architecture overview of \method{}.
}
\label{fig:framework}
\end{center}
\vskip -0.2in
\end{figure*}
\section{Preliminary}
\label{sec:prelim}

\subsection{Problem Formulation}
\begin{definition}[POI]
A point of interest (POI) is a place $i$ associated with geographic coordinates $\boldsymbol{\ell}_i=(\mathrm{lat}_i,\mathrm{lon}_i)$ and a category $c_i\in\mathcal{K}$, where $\mathcal{K}$ is the category vocabulary. The POIs of a city form a set $\mathcal{P}=\{p_1,\dots,p_N\}$ of size $N$, which we connect into a graph $\mathcal{G}=(\mathcal{P},\mathcal{E})$ whose edges $\mathcal{E}$ link spatially or semantically related POIs.
\end{definition}

\begin{definition}[Trajectory]
A check-in $(p,t)$ records a user's visit to POI $p$ at timestamp $t$. A trajectory is a discrete, time-ordered sequence $\tau=\big((p_1,t_1),\dots,(p_L,t_L)\big)$ of $L$ check-ins, with $t_1\le t_L$.
\end{definition}

\begin{definition}[Target-trajectory-free Generation]
Let $\mathcal{S}$ index the source cities. For each source city $s\in\mathcal{S}$ we observe its POI set $\mathcal{P}^{s}$ and a set of real trajectories $\mathcal{T}^{s}$ over $\mathcal{P}^{s}$, whereas for the target city we observe only its POI set $\mathcal{P}^{\mathrm{tgt}}$. The task is to learn a generator $f(\cdot)$: 
\begin{equation}
\widehat{\mathcal{T}}^{\mathrm{tgt}}
= f\Big(\{(\mathcal{P}^{s},\mathcal{T}^{s})\}_{s\in\mathcal{S}},\ \mathcal{P}^{\mathrm{tgt}}\Big),
\label{eq:ttf}
\end{equation}
whose output $\widehat{\mathcal{T}}^{\mathrm{tgt}}$ is a set of check-in trajectories over $\mathcal{P}^{\mathrm{tgt}}$ that is distributionally close to the real trajectories of the target city. The defining constraint is that the target trajectory set $\mathcal{T}^{\mathrm{tgt}}$, and any statistic derived from it. This information boundary separates the problem from domain adaptation and few-shot transfer learning, where target observations remain available.
\label{def:ttf}
\end{definition}

\subsection{Transferability of Transitions}

Definition~\ref{def:ttf} forbids any target trajectory, so $f$ can rely only on structure that is shared across cities. We argue that this shared structure lies in the transitions between places rather than in the absolute places visited. and we accordingly read a trajectory as a chain of transitions.

\begin{definition}[Transition]
A transition is the move between two consecutive check-ins $(p_{k-1},t_{k-1})$ and
$(p_k,t_k)$. We describe it by three quantities that do not name an absolute
destination, namely a representation of the change of place context, the ground
displacement between the two POIs, and the elapsed time $\Delta t=t_k-t_{k-1}$. A
trajectory is then a chain of such transitions.
\label{def:transition}
\end{definition}

The task above transfer what is learned from source trajectories to a city that has none. Whether the transition view supports such transfer needs justification. Based on the classical invariance theory \cite{lehmann2005testing, ben2010theory}, we now characterize which transition representations transfer across cities, and show that the displacement between POI embeddings is the canonical one. 
To begin, we establish a translation alignment assumption for this problem,
Let $\mathbf{h}$ be any map that embeds every POI of every city into a shared space $\mathcal{H}$.

\begin{assumption}[Translation Alignment]
For two cities $A$ and $B$, let $\varphi$ be a functional correspondence that maps a POI of $A$ to a POI of $B$ playing the same functional role. There exists an offset $\mathbf{t}_{AB}\in\mathcal{H}$ such that $\mathbf{h}_{\varphi(p)}=\mathbf{h}_{p}+\mathbf{t}_{AB}$ for all $p\in\mathcal{P}_{A}$.
\label{asm:gauge}
\end{assumption}
That is, the cities populate the shared space as copies of one functional layout that differ only by a shift, in the spirit of relations as embedding differences \cite{ben2010theory, bordes2013translating}. This is a requirement on the learned representation rather than on raw data. To satisfy this assumption, we introduced a cross-city shared encoder in Sec. ~\ref{sec:encoder}. 
With the space aligned up to a shift, we next define what it means for a transition representation to transfer across cities.

\begin{definition}[Transferable Transition Representation]
A transition representation $F$ assigns a vector $F(\mathbf{h}_{a},\mathbf{h}_{b})$ to the move from POI $a$ to $b$. It is transferable if functionally equivalent moves receive the same value in every city:
\begin{equation}
F\big(\mathbf{h}_{\varphi(a)},\,\mathbf{h}_{\varphi(b)}\big)=F\big(\mathbf{h}_{a},\,\mathbf{h}_{b}\big).
\label{eq:transferable}
\end{equation}
\label{def:transferable}
\end{definition}
Under Assumption~\ref{asm:gauge} this is exactly invariance to a common shift of both arguments, and the classical factorization of invariant functions \cite{lehmann2005testing} yields the following:

\begin{theorem}
Under Assumption~\ref{asm:gauge}, the displacement $\boldsymbol{\delta}(a,b)=\mathbf{h}_{b}-\mathbf{h}_{a}$ is transferable, and it is complete in that every transferable transition representation is a function of $\boldsymbol{\delta}$. Hence, whenever the assumption holds, a transition distribution fitted on source cities applies to the functionally same moves of any other city without target data.
\label{thm:delta}
\end{theorem}
The proof is deferred to Appendix~\ref{app:proof_d}. We stress that the guarantee is conditional. It does not assert that real cities admit such a representation, but that once the encoder aligns them up to a shift, modeling $\boldsymbol{\delta}$ preserves all transferable information. The theorem also delimits what $\boldsymbol{\delta}$ alone cannot carry.

\section{Methodology}
\label{sec:method}
In this section, we present a three-stage framework \textit{\method{}} for target-trajectory-free generation, as illustrated in Fig. \ref{fig:framework}.

\subsection{Cross-City POI Alignment}
\label{sec:encoder}

The goal of this stage is to approximate Assumption ~\ref{asm:gauge}, which embeds the POIs of every city into one shared semantic space.
What defines the role of a POI is not its identity but its attributes and the places around it. We therefore build the POIs of each city into a multi-relational graph, which edges are induced by spatial location, and node features are derived from the category semantics and geographic coordinates.
Finally, we encode this graph with a Relation-aware Graph Attention Network to capture the shared representation underlying functionally similar POIs across cities.

\textbf{Multi-relational POIs Graph.}
Specifically, We connect the POIs of each city with three relations. We first partition the city into hexagonal cells, following the standard way in spatial analysis  \cite{wozniak2021hex2vec}. Each cell covers a compact local area of roughly one square kilometer. Within a cell, we link each POI to its nearest neighbors, so that places close to each other are connected. Across neighboring cells, we link each POI to its nearest neighbors, so that its wider spatial surroundings are connected. For each POI, we also link the nearest POIs of the same category, so that places of similar role are connected. All three kinds of edges are built by nearest-neighbor search over geographic distance.

The input feature of a POI is made of its coordinates and category semantics. We obtain the semantics by encoding the category name with Sentence-BERT \cite{reimers2019sentence}. 
We write the feature of POI $i$ as $\mathbf{x}_i=[\,\bar{\boldsymbol{\ell}}_i\,\|\,\mathbf{s}_{c_i}\,]$, where $\bar{\boldsymbol{\ell}}_i$ is the normalized coordinates and $\mathbf{s}_{c_i}$ is the category embedding. 
For each relation $r$, a GATv2 \cite{brody2021attentive} turns the neighbors of a POI into a per-relation representation $\mathbf{h}^r$. It takes the attention-weighted sum of the neighbor features as follows:
\begin{equation}
\begin{aligned}
e_{ij}^{r}&=\mathbf{a}_r^{\top}\sigma\big(\mathbf{W}_r[\mathbf{x}_i\,\|\,\mathbf{x}_j]\big),\\
\alpha_{ij}^{r}&=\frac{\exp(e_{ij}^{r})}{\sum_{k\in\mathcal{N}_r(i)}\exp(e_{ik}^{r})},\\
\mathbf{h}_i^{r}&=\sigma\Big(\sum\nolimits_{j\in\mathcal{N}_r(i)}\alpha_{ij}^{r}\,\mathbf{W}_r\,\mathbf{x}_j\Big),
\end{aligned}
\label{eq:gatv2}
\end{equation}
where $\mathbf{W}_r$ and $\mathbf{a}_r$ are the projection and attention vector of relation $r$, $\sigma$ is a nonlinearity. An aggregation layer then merges the per-relation representations $\mathbf{h}^r$ into the updated POI embedding,
\begin{equation}
\mathbf{h}_i = \sum\nolimits_{r}(\beta_r\,\mathbf{h}_i^{r}),
\label{eq:enclayer}
\end{equation}
where $\beta$ is the learnable weight. We train the encoder jointly on the graphs of multiple source cities, so that it learns to align functionally similar POIs across cities.

\textbf{Multi-task Alignment Objective.}
To drive the POI representations toward cross-city alignment, we optimize the encoder with three tasks at different levels.
\uline{At the node level}, a classification head $\mathbf{W}_{cat}$ predicts the category of each POI from its embedding. On a masked subset $\mathcal{M}$ of nodes, whose category features are hidden, we minimize the cross-entropy:
\begin{equation}
\mathcal{L}_{\mathrm{cat}}=\frac{1}{|\mathcal{M}|}\sum_{i\in\mathcal{M}}\mathcal{C}\big(c_i,\,\mathbf{W}_{\mathrm{cat}}\mathbf{h}_i\big),
\label{eq:catloss}
\end{equation}
where $\mathcal{C}(\cdot)$ is the cross-entropy function, $c_i$ is the true category of node $i$.
\uline{At the link level}, a link prediction loss recovers masked edges and keeps the embedding geometry consistent with the graph. For each masked edge $(u,v)$ and a sampled non-edge $(u,v^{-})$, we minimize the cross-entropy:
\begin{equation}
\mathcal{L}_{\mathrm{link}}=\frac{1}{|\mathcal{E}_m|}\sum_{(u,v)\in\mathcal{E}_m}\Big[\mathcal{C}\big(1,\sigma(\mathbf{h}_u^{\top}\mathbf{h}_v)\big)+\mathcal{C}\big(0,\sigma(\mathbf{h}_u^{\top}\mathbf{h}_{v^{-}})\big)\Big],
\label{eq:linkloss}
\end{equation}
where $\mathcal{E}_m$ is the set of masked edges, the label is $1$ for a true edge $(u,v)$ and $0$ for a sampled non-edge $(u,v^{-})$. $\sigma(\cdot)$ is the activation function.
\uline{At the cross-city level}, a contrastive loss pulls each POI toward a place of a similar category in another city while pushing it away from POIs of distinct categories:
\begin{equation}
\small
\mathcal{L}_{\mathrm{con}}=-\frac{1}{|\mathcal{A}|}\sum_{a\in\mathcal{A}}
\log\frac{\exp\!\big(\mathrm{sim}(\mathbf{h}_a,\mathbf{h}^{+})/\tau\big)}
{\exp\!\big(\mathrm{sim}(\mathbf{h}_a,\mathbf{h}^{+})/\tau\big)+\sum_{j}\exp\!\big(\mathrm{sim}(\mathbf{h}_a,\mathbf{h}^{-}_j)/\tau\big)},
\label{eq:conloss}
\end{equation}
where $\mathbf{h}^{+}$ and $\{\mathbf{h}^{-}_j\}$ are the positive and negatives drawn from pre-built sets. The $\mathrm{sim}(\cdot)$ is cosine similarity, and $\tau$ is a temperature.

Consequently, the overall loss function for the cross-city encoder is formulated as follows:
\begin{equation}
\mathcal{L}_{\mathrm{align}}=\mathcal{L}_{\mathrm{con}}+\lambda_{\mathrm{cat}}\mathcal{L}_{\mathrm{cat}}+\lambda_{\mathrm{link}}\mathcal{L}_{\mathrm{link}},
\label{eq:encloss}
\end{equation}
where $\lambda_{cat}$ and $\lambda_{link}$ are the balanced parameters.
The aggregated GAT describes a POI by its spatial and categorical neighborhood, so functionally similar places acquire similar embeddings in cities that share no POI. The multi-task objective supplies the training signal that shapes this space. The category and link tasks make an embedding encode what a POI is and how it sits among its neighbors. The contrastive task then draws the same function from different cities into one region, so that the per-city embeddings become a single city-agnostic space. This is the precondition for extracting human movement knowledge across cities.

\subsection{Flow Matching in the Transition Space}
\label{sec:fm}

With all POIs placed in one shared space, this stage now model how people move within it. We construct the transition vector, and learn its distribution with a flow-matching-based generator from which new transitions can be sampled.

\textbf{Transition Vector.}
The object we model is the transition vector, which describes a single move between two consecutive check-ins. For two consecutive check-ins $(p_{k-1},p_k)$ we form it by subtracting the two POI embeddings in the shared space and appending a temporal and a geographic channel:
\begin{equation}
\boldsymbol{\delta}_k=\big[\,(\mathbf{h}_{p_k}-\mathbf{h}_{p_{k-1}})\ \big\|\ \phi_{\mathrm{time}}(\Delta t_k)\ \big\|\ \phi_{\mathrm{geo}}(\Delta\boldsymbol{\ell}_k)\,\big],
\label{eq:delta}
\end{equation}
where $\Delta t_k$ is the time gap. $\Delta\boldsymbol{\ell}_k$ is the ground displacement. $\phi_{\mathrm{time}},\phi_{\mathrm{geo}}$ encode them separately. Because every channel is a difference rather than an absolute value, $\boldsymbol{\delta}_k$ records what kind of move was made, not where it happened, and is therefore transferable across cities.
By Theorem~\ref{thm:delta}, this choice is canonical rather than heuristic. The difference is transferable, so modeling $\boldsymbol{\delta}$ loses nothing.

\textbf{History-conditioned Training.}
Another crucial design of this stage is how a source trajectory becomes training samples. 

Each sample can be a single transition $\boldsymbol{\delta}_k$. Yet according to Corollary~\ref{cor:context}, $\boldsymbol{\delta}$ is anti-symmetric and blind to its anchor, so it cannot express where a move starts, which direction it prefers, or whether it returns to a familiar place. We therefore further condition the model on the trajectory history.
Concretely, we supply the sequence of past POI embeddings as a context $\mathbf{c}_k=(\mathbf{h}_{p_1},\dots,\mathbf{h}_{p_{k-1}})$ and apply a mask that hides every later check-in, so the model predicts the $k$-th move from the first $k{-}1$ steps. In this way one trajectory yields a transition sample at every step, while at generation the same network draws a transition from a single conditioning POI.

\textbf{Flow-matching Generator.}
Given these samples we introduce a standard flow-matching generator \cite{lipman2022flow,tong2023improving}. It draws noise $\boldsymbol{\delta}^{0}\sim\mathcal{N}(\mathbf{0},\mathbf{I})$, interpolates it with the real transition $\boldsymbol{\delta}^{1}=\boldsymbol{\delta}_k$ along the straight path $\boldsymbol{\delta}^{t}=(1-t)\boldsymbol{\delta}^{0}+t\boldsymbol{\delta}^{1}$, and trains a velocity network $v_\theta$ conditioned on $\mathbf{c}_k$ to regress the transport velocity:
\begin{equation}
\mathcal{L}_{\mathrm{fm}}=\mathbb{E}_{t,\boldsymbol{\delta}^{0},\boldsymbol{\delta}^{1}}\big\|\,v_\theta(\boldsymbol{\delta}^{t},t,\mathbf{c}_k)-(\boldsymbol{\delta}^{1}-\boldsymbol{\delta}^{0})\,\big\|_2^2 .
\label{eq:fmloss}
\end{equation}

We instantiate this velocity network $v_\theta$ as a Transformer that runs over the whole trajectory in one pass following \cite{peebles2023scalable,ma2024sit}. The input is the noised transitions, all of POI embeddings and the flow time $t$. Its output is the predicted velocity $v_\theta(\boldsymbol{\delta}^{t},t,\mathbf{c}_k)$. 
Each token carries only the step $k$, formed by projecting its noised transition $\boldsymbol{\delta}_k^{t}$ and the $k$-th POI embedding into a latent space. A learnable positional embedding $\mathbf{PE}_k$ marks the step's order:
\begin{equation}
\mathbf{u}_k=[\boldsymbol{\delta}_k^{t}\,\|\,\mathbf{h}_{p_{k-1}}]+\mathbf{PE}_k.
\label{eq:token}
\end{equation}

The complete history of each trajectory is incorporated into the training process by a standard self-attention mechanism with an auto-regressive mask. Step $k$ attends only to the tokens at positions $k'\le k$, which lets it read the whole context $\mathbf{c}_k$ while never seeing a later step. The flow time $t$ is injected by modulating every block with adaptive layer normalization, following the time conditioning of DiT \cite{peebles2023scalable} and the feature-wise affine modulation of FiLM \cite{perez2018film}.

During the generation, a transition is drawn by integrating the learned velocity from noise. Starting from $\boldsymbol{\delta}^{0}$ with a conditioning context $\mathbf{c}$, we solve the flow ODE from $t{=}0$ to $t{=}1$ with $N$ forward-Euler steps of size $\Delta t=1/N$,
\begin{equation}
\boldsymbol{\delta}^{t+\Delta t}=\boldsymbol{\delta}^{t}+\Delta t\,v_\theta(\boldsymbol{\delta}^{t},t,\mathbf{c}),\qquad t=0,\tfrac{1}{N},\dots,\tfrac{N-1}{N}.
\label{eq:sample}
\end{equation}

The endpoint $\hat{\boldsymbol{\delta}}=\boldsymbol{\delta}^{1}$ is the generated transition. Setting the context $\mathbf{c}$ to a target-city POI yields the sampler that the next stage grounds into trajectories.
By generating in the transition space rather than over absolute locations, this stage decouples the pattern of movement from the specific places a city contains. The generator learns only how moves unfold and carries none of the source city's POIs, so it can transfer to an unseen city without retraining.

\subsection{Grounding in Unseen Cities}
\label{sec:generation}
Here the generator produces continuous transitions, but a trajectory generally is a sequence of discrete POIs of the target city. We bridge this gap in two steps: an induction step that turns the generated samples into a behavior graph over the target POIs, and a grounding step that walks this graph into trajectories.

\textbf{Behavior Graph Induction.}
For the first step, we build a static behavior graph before grounding. For every target POI $i$, we condition the frozen generator on $\mathbf{h}_i$ and draw $K$ transitions $\{\hat{\boldsymbol{\delta}}^{(s)}\}_{s=1}^{K}$. Each sampled transition is snapped to the target POI whose own displacement from $i$ best matches it in both the semantic and the geographic channel:
\begin{equation}
j^{\star}(\hat{\boldsymbol{\delta}})=\arg\max_{j}\ \cos\!\big(\hat{\boldsymbol{\delta}}^{\mathrm{sem}},\,(\mathbf{h}_j-\mathbf{h}_i)\big)+\,\cos\!\big(\hat{\boldsymbol{\delta}}^{\mathrm{geo}},\,\phi_{\mathrm{geo}}(\Delta\boldsymbol{\ell}_{ij})\big),
\label{eq:snap}
\end{equation}
where $\hat{\boldsymbol{\delta}}^{\mathrm{sem}}$ and $\hat{\boldsymbol{\delta}}^{\mathrm{geo}}$ are the semantic and geographic part of $\hat{\boldsymbol{\delta}}$. The candidates are restricted to POIs of the target city. Counting how often each successor is chosen over the $K$ draws gives a transition probability $P(i\!\to\! j)$:
\begin{equation}
P(i\!\to\! j)=\frac{1}{K}\sum_{s=1}^{K}\!\big[\,j^{\star}(\hat{\boldsymbol{\delta}}^{(s)})=j\,\big].
\label{eq:edgeprob}
\end{equation}
By retaining the high-probability successors of each POI, we can obtain a directed weighted behavior graph $\mathcal{G}_B$ over the target POIs.

\textbf{Pointer Network Grounding.}
For the second step, we will walk $\mathcal{G}_B$ into a trajectory. Exploring the graph alone would only ever move to new places, but real mobility also returns to familiar ones \cite{song2010modelling}.
Therefore, at each step the walk should decide whether to return and where to return. 
By Corollary~\ref{cor:context}, such a return preference is history dependent and cannot live in $\boldsymbol{\delta}$ itself, so it may be decided by a module conditioned on the walk history.
We introduce a lightweight Pointer Network \cite{see2017get, vinyals2015pointer}, pre-trained on the source city trajectories in which every step is labeled as a return or a new move. At a step $k$ we form a contextual feature vector:
\begin{equation}
\boldsymbol{\xi}_k=\Big[\ \sin\tfrac{2\pi\,\mathrm{hod}_k}{24},\ \ \cos\tfrac{2\pi\,\mathrm{hod}_k}{24},\ \ \tfrac{k}{L},\ \ \mathrm{dist}(\boldsymbol{\ell}_{p_k},\boldsymbol{\ell}_{p_1})\ \Big],
\label{eq:ptrfeat}
\end{equation}
where $\mathrm{hod}_k\in[0,24)$ is the hour of day encoded cyclically by the first two entries, $k/L$ is the normalized step index, and the last entry is the haversine distance from the current POI to the start POI $p_1$. Concatenating these features with the current POI embedding, a two-layer MLP encodes them into a walk state $\mathbf{st}_k=\mathrm{MLP}([\mathbf{h}_{p_k}\,\|\,\boldsymbol{\xi}_k])$. 
From this state we use two heads to read the decisions. A gate decides whether to return and an attention score decides where to return:
\begin{equation}
\begin{aligned}
\mathrm{prob}_k &= \sigma\!\big(\mathbf{w}_g^{\top}\mathbf{st}_k\big),\\
\mathrm{score}_k(j) &= \mathrm{softmax}_{j<k}\!\big(\mathbf{st}_k^{\top}\mathbf{w}_a\,\mathbf{h}_{p_j}\big),
\end{aligned}
\label{eq:ptrheads}
\end{equation}
where $\mathrm{prob}_k$ is the probability of returning at step $k$, $\mathrm{score}_k$ is the resulting distribution over which past POI to revisit. $\mathbf{w}_g$ and $\mathbf{w}_a$ are the learnable weights, $\sigma$ is the activation function. Finally, We draw a return with probability $\mathrm{prob}_k$, and on a return we sample the revisited POI from $\mathrm{score}_k$. Otherwise the walk explores, sampling the next POI from the out-edges $P(p_k\!\to\!\cdot)$ of $\mathcal{G}_B$. 
The network is trained on the source trajectories, jointly by a binary loss on the gate against the return label and a cross-entropy on the attention head against the revisited POI on the return steps. Additionally, we introduce a trade-off coefficient $\lambda_{gate}$ to combine the two losses and jointly train the Pointer Network. Because the balance between returning and exploring is a mobility regularity that holds across cities \cite{song2010modelling,pappalardo2015returners}, the return rate the gate learns on the source transfers as a prior rather than requiring any target data.

Together these two steps, we can assemble a full trajectory. We first sample a random timestamp and a start POI from the target city. At each step the Pointer Network follows the decisions and the behavior graph to the next place. The coordinate of the step is the coordinate of the visited POI. The time gap comes from the inverse of transition. We repeat until the sampled length. Finally, we can obtain a full synthetic trajectory of the target city's own POIs, with coordinates and times. Please refer to Appendix \ref{algorithm} for detailed procedures of the proposed algorithm, corresponding computational analysis and running time comparison.

\begin{table*}[ht]
\centering
\small
\caption{Main results on four zero-shot transfer settings. Each value is the mean$\pm$std over five generation seeds. \textbf{Bold} and \uline{underline} mean the best and the second-best results. \textbf{Markov}'s outputs do not include timestamps, so `--' is used as a placeholder. The last column reports \method{}'s average relative gain over the second-best method across the eight metrics.}
\label{tab:main}
\begin{tabular}{ll cccc cccc c}
\toprule
& \multirow{2}{*}{Methods} & \multicolumn{4}{c}{Downstream Utility} & \multicolumn{4}{c}{Distributional Fidelity}  &  \\
\cmidrule(lr){3-6}\cmidrule(lr){7-10}
& & Pred-MAPE\,$\downarrow$ & Pred-MSPE\,$\downarrow$ & Rec-MAPE\,$\downarrow$ & Rec-MSPE\,$\downarrow$ & OD-Sim\,$\downarrow$ & Distance\,$\downarrow$ & Speed\,$\downarrow$ & Time-Cat\,$\downarrow$  & \\
\midrule
\multirow{7}{*}{\rotatebox{90}{\footnotesize NYC$\to$TKY}} & Markov \cite{gambs2012next} & 0.390$\pm$0.040 & 0.165$\pm$0.036 & 0.180$\pm$0.049 & 0.054$\pm$0.037 & 0.616$\pm$0.005 & 0.227$\pm$0.001 & -- & --  & \multirow{7}{*}{\rotatebox{90}{\textcolor{red}{+24.76\%}}} \\
 & DITRAS \cite{pappalardo2018data} & 0.396$\pm$0.030 & 0.181$\pm$0.030 & 0.276$\pm$0.068 & 0.111$\pm$0.043 & 0.516$\pm$0.007 & 0.181$\pm$0.001 & 0.077$\pm$0.006 & 0.295$\pm$0.002  & \\
 & Gravity \cite{barbosa2018human} & 0.419$\pm$0.049 & 0.193$\pm$0.045 & 0.251$\pm$0.132 & 0.119$\pm$0.112 & \uline{0.460$\pm$0.003} & \uline{0.126$\pm$0.001} & 0.117$\pm$0.004 & \uline{0.265$\pm$0.001}  & \\
 & MIRAGE \cite{deng2025revisiting} & 0.390$\pm$0.015 & 0.188$\pm$0.010 & 0.197$\pm$0.004 & 0.110$\pm$0.001 & 0.925$\pm$0.068 & 0.137$\pm$0.001 & \uline{0.055$\pm$0.001} & 0.273$\pm$0.002  & \\
 & GTG \cite{wang2025gtg} & 0.392$\pm$0.015 & 0.191$\pm$0.010 & 0.168$\pm$0.025 & \uline{0.020$\pm$0.010} & 0.593$\pm$0.018 & 0.236$\pm$0.001 & 0.186$\pm$0.001 & 0.341$\pm$0.001  & \\
 & COLA \cite{wang2024cola} & \uline{0.342$\pm$0.023} & \uline{0.127$\pm$0.023} & \uline{0.152$\pm$0.032} & 0.035$\pm$0.016 & 0.656$\pm$0.004 & 0.275$\pm$0.002 & 0.193$\pm$0.001 & 0.335$\pm$0.001  & \\
 & \textbf{\method{}} & \textbf{0.318$\pm$0.006} & \textbf{0.107$\pm$0.002} & \textbf{0.130$\pm$0.003} & \textbf{0.019$\pm$0.003} & \textbf{0.393$\pm$0.007} & \textbf{0.036$\pm$0.005} & \textbf{0.023$\pm$0.007} & \textbf{0.234$\pm$0.011}  & \\
\cmidrule(l){1-11}
\multirow{7}{*}{\rotatebox{90}{\footnotesize KYO+NGO$\to$OSA}} & Markov \cite{gambs2012next} & 0.394$\pm$0.038 & 0.193$\pm$0.031 & 0.333$\pm$0.037 & 0.173$\pm$0.036 & 0.441$\pm$0.009 & 0.210$\pm$0.003 & -- & --  & \multirow{7}{*}{\rotatebox{90}{\textcolor{red}{+33.55\%}}} \\
 & DITRAS \cite{pappalardo2018data} & 0.490$\pm$0.092 & 0.274$\pm$0.093 & 0.420$\pm$0.112 & 0.251$\pm$0.122 & \uline{0.321$\pm$0.008} & 0.169$\pm$0.002 & 0.055$\pm$0.001 & 0.398$\pm$0.002  & \\
 & Gravity \cite{barbosa2018human} & 0.419$\pm$0.031 & 0.224$\pm$0.015 & 0.378$\pm$0.123 & 0.240$\pm$0.120 & 0.328$\pm$0.006 & 0.141$\pm$0.002 & \uline{0.041$\pm$0.001} & 0.310$\pm$0.002  & \\
 & MIRAGE \cite{deng2025revisiting} & \uline{0.381$\pm$0.045} & \uline{0.171$\pm$0.028} & 0.290$\pm$0.079 & 0.161$\pm$0.078 & 0.856$\pm$0.036 & \uline{0.110$\pm$0.002} & 0.050$\pm$0.001 & 0.329$\pm$0.003  & \\
 & GTG \cite{wang2025gtg} & 0.423$\pm$0.027 & 0.222$\pm$0.020 & \uline{0.237$\pm$0.033} & 0.126$\pm$0.010 & 0.638$\pm$0.015 & 0.174$\pm$0.002 & 0.186$\pm$0.001 & \uline{0.258$\pm$0.004}  & \\
 & COLA \cite{wang2024cola} & 0.434$\pm$0.021 & 0.206$\pm$0.013 & 0.259$\pm$0.032 & \uline{0.108$\pm$0.016} & 0.498$\pm$0.012 & 0.254$\pm$0.001 & 0.126$\pm$0.002 & 0.271$\pm$0.001  & \\
 & \textbf{\method{}} & \textbf{0.372$\pm$0.004} & \textbf{0.148$\pm$0.026} & \textbf{0.149$\pm$0.018} & \textbf{0.036$\pm$0.013} & \textbf{0.232$\pm$0.009} & \textbf{0.046$\pm$0.012} & \textbf{0.023$\pm$0.003} & \textbf{0.209$\pm$0.002}  & \\
\cmidrule(l){1-11}
\multirow{7}{*}{\rotatebox{90}{\footnotesize FUK+KYO+NGO$\to$OSA}} & Markov \cite{gambs2012next} & \uline{0.372$\pm$0.038} & \uline{0.157$\pm$0.019} & 0.262$\pm$0.099 & 0.103$\pm$0.077 & 0.436$\pm$0.014 & 0.206$\pm$0.001 & -- & --  & \multirow{7}{*}{\rotatebox{90}{\textcolor{red}{+16.44\%}}} \\
 & DITRAS \cite{pappalardo2018data} & 0.419$\pm$0.069 & 0.233$\pm$0.075 & 0.322$\pm$0.124 & 0.166$\pm$0.114 & 0.306$\pm$0.016 & 0.164$\pm$0.002 & 0.053$\pm$0.001 & 0.296$\pm$0.002  & \\
 & Gravity \cite{barbosa2018human} & 0.421$\pm$0.031 & 0.200$\pm$0.043 & 0.246$\pm$0.091 & 0.117$\pm$0.092 & \uline{0.275$\pm$0.006} & 0.142$\pm$0.004 & \uline{0.030$\pm$0.001} & 0.306$\pm$0.002  & \\
 & MIRAGE \cite{deng2025revisiting} & 0.384$\pm$0.024 & 0.210$\pm$0.017 & 0.286$\pm$0.070 & 0.157$\pm$0.072 & 0.847$\pm$0.146 & \uline{0.135$\pm$0.001} & 0.044$\pm$0.001 & 0.345$\pm$0.003  & \\
 & GTG \cite{wang2025gtg} & 0.429$\pm$0.030 & 0.228$\pm$0.023 & \uline{0.194$\pm$0.043} & \uline{0.059$\pm$0.026} & 0.645$\pm$0.022 & 0.175$\pm$0.001 & 0.186$\pm$0.001 & \uline{0.263$\pm$0.004}  & \\
 & COLA \cite{wang2024cola} & 0.393$\pm$0.041 & 0.178$\pm$0.026 & 0.222$\pm$0.104 & 0.090$\pm$0.083 & 0.502$\pm$0.016 & 0.256$\pm$0.001 & 0.130$\pm$0.001 & 0.274$\pm$0.002  & \\
 & \textbf{\method{}} & \textbf{0.367$\pm$0.026} & \textbf{0.150$\pm$0.017} & \textbf{0.192$\pm$0.054} & \textbf{0.053$\pm$0.035} & \textbf{0.259$\pm$0.006} & \textbf{0.050$\pm$0.002} & \textbf{0.023$\pm$0.001} & \textbf{0.204$\pm$0.002}  & \\
\cmidrule(l){1-11}
\multirow{7}{*}{\rotatebox{90}{\footnotesize ANK+BUR+IZM$\to$IST}} & Markov \cite{gambs2012next} & 0.348$\pm$0.102 & 0.165$\pm$0.087 & 0.332$\pm$0.057 & 0.205$\pm$0.085 & 0.657$\pm$0.009 & 0.166$\pm$0.002 & -- & --  & \multirow{7}{*}{\rotatebox{90}{\textcolor{red}{+20.76\%}}} \\
 & DITRAS \cite{pappalardo2018data} & 0.321$\pm$0.033 & 0.159$\pm$0.043 & 0.333$\pm$0.172 & 0.197$\pm$0.151 & 0.551$\pm$0.016 & 0.133$\pm$0.002 & 0.048$\pm$0.004 & \uline{0.210$\pm$0.001}  & \\
 & Gravity \cite{barbosa2018human} & 0.279$\pm$0.102 & 0.123$\pm$0.071 & 0.243$\pm$0.116 & 0.115$\pm$0.089 & \uline{0.408$\pm$0.007} & \uline{0.088$\pm$0.002} & \uline{0.027$\pm$0.001} & 0.238$\pm$0.001  & \\
 & MIRAGE \cite{deng2025revisiting} & \uline{0.209$\pm$0.063} & 0.076$\pm$0.063 & 0.335$\pm$0.109 & 0.259$\pm$0.126 & 0.963$\pm$0.015 & 0.158$\pm$0.003 & 0.044$\pm$0.002 & 0.323$\pm$0.003  & \\
 & GTG \cite{wang2025gtg} & 0.218$\pm$0.012 & \uline{0.070$\pm$0.006} & \uline{0.118$\pm$0.040} & \uline{0.026$\pm$0.015} & 0.679$\pm$0.018 & 0.316$\pm$0.003 & 0.164$\pm$0.001 & 0.228$\pm$0.003  & \\
 & COLA \cite{wang2024cola} & 0.230$\pm$0.050 & 0.071$\pm$0.039 & 0.164$\pm$0.086 & 0.063$\pm$0.075 & 0.824$\pm$0.012 & 0.265$\pm$0.002 & 0.206$\pm$0.002 & 0.226$\pm$0.001  & \\
 & \textbf{\method{}} & \textbf{0.195$\pm$0.047} & \textbf{0.056$\pm$0.023} & \textbf{0.117$\pm$0.007} & \textbf{0.023$\pm$0.005} & \textbf{0.399$\pm$0.005} & \textbf{0.034$\pm$0.001} & \textbf{0.010$\pm$0.013} & \textbf{0.209$\pm$0.001}  & \\
\bottomrule
\end{tabular}
\end{table*}

\section{Experiments}
\label{sec:exp}
\subsection{Experimental Settings}

\begin{table*}[t]
\centering
\small
\caption{Ablation study on the four transfers. Each value is the mean$\pm$std over five generation seeds. \textbf{Bold} mean the best results.}
\label{tab:ablation}
\begin{tabular}{ll cccc cccc}
\toprule
& Variants & \multicolumn{4}{c}{Downstream Utility} & \multicolumn{4}{c}{Distributional Fidelity} \\
\cmidrule(lr){3-6}\cmidrule(lr){7-10}
& & Pred-MAPE\,$\downarrow$ & Pred-MSPE\,$\downarrow$ & Rec-MAPE\,$\downarrow$ & Rec-MSPE\,$\downarrow$ & OD-Sim\,$\downarrow$ & Distance\,$\downarrow$ & Speed\,$\downarrow$ & Time-Cat\,$\downarrow$ \\
\midrule
\multirow{5}{*}{\rotatebox{90}{\scriptsize NYC$\to$TKY}} & w/o graph encoder & 0.419$\pm$0.004 & 0.247$\pm$0.003 & 0.224$\pm$0.010 & 0.117$\pm$0.001 & 0.501$\pm$0.002 & 0.060$\pm$0.001 & 0.025$\pm$0.000 & 0.254$\pm$0.001 \\
 & w/o flow-matching & 0.484$\pm$0.016 & 0.224$\pm$0.014 & 0.141$\pm$0.015 & 0.053$\pm$0.004 & 0.513$\pm$0.004 & 0.148$\pm$0.000 & 0.037$\pm$0.000 & 0.249$\pm$0.000 \\
 & w/o copy-gate & 0.431$\pm$0.005 & 0.178$\pm$0.001 & 0.160$\pm$0.000 & 0.042$\pm$0.003 & 0.400$\pm$0.002 & 0.049$\pm$0.000 & 0.025$\pm$0.000 & 0.235$\pm$0.000 \\
 & w/o AR-mask & 0.427$\pm$0.016 & 0.148$\pm$0.015 & 0.156$\pm$0.011 & 0.076$\pm$0.002 & 0.437$\pm$0.002 & 0.070$\pm$0.001 & 0.026$\pm$0.001 & 0.265$\pm$0.001 \\
 & w/o condition & 0.409$\pm$0.007 & 0.161$\pm$0.005 & 0.150$\pm$0.002 & 0.084$\pm$0.001 & 0.465$\pm$0.005 & 0.133$\pm$0.000 & 0.026$\pm$0.000 & 0.255$\pm$0.001 \\
 & \textbf{\method{}} & \textbf{0.318$\pm$0.006} & \textbf{0.107$\pm$0.002} & \textbf{0.130$\pm$0.003} & \textbf{0.019$\pm$0.003} & \textbf{0.393$\pm$0.007} & \textbf{0.036$\pm$0.005} & \textbf{0.023$\pm$0.007} & \textbf{0.234$\pm$0.011} \\
\cmidrule(l){1-10}

\multirow{5}{*}{\rotatebox{90}{\scriptsize KYO+NGO$\to$OSA}} & w/o graph encoder & 0.518$\pm$0.021 & 0.303$\pm$0.014 & 0.225$\pm$0.013 & 0.123$\pm$0.004 & 0.290$\pm$0.008 & 0.075$\pm$0.002 & 0.034$\pm$0.001 & 0.216$\pm$0.001 \\
 & w/o flow-matching & 0.410$\pm$0.035 & 0.230$\pm$0.022 & 0.245$\pm$0.011 & 0.088$\pm$0.012 & 0.312$\pm$0.005 & 0.109$\pm$0.001 & 0.043$\pm$0.001 & 0.219$\pm$0.001 \\
 & w/o copy-gate & 0.489$\pm$0.038 & 0.254$\pm$0.025 & 0.309$\pm$0.068 & 0.161$\pm$0.086 & 0.261$\pm$0.004 & 0.054$\pm$0.001 & 0.024$\pm$0.000 & 0.210$\pm$0.001 \\
 & w/o AR-mask & 0.452$\pm$0.017 & 0.246$\pm$0.011 & 0.196$\pm$0.060 & 0.062$\pm$0.030 & 0.238$\pm$0.016 & 0.061$\pm$0.001 & 0.029$\pm$0.001 & 0.255$\pm$0.001 \\
 & w/o condition & 0.440$\pm$0.021 & 0.245$\pm$0.009 & 0.223$\pm$0.132 & 0.079$\pm$0.069 & 0.260$\pm$0.011 & 0.065$\pm$0.000 & 0.033$\pm$0.001 & 0.251$\pm$0.001 \\
 & \textbf{\method{}} & \textbf{0.372$\pm$0.004} & \textbf{0.148$\pm$0.026} & \textbf{0.149$\pm$0.018} & \textbf{0.036$\pm$0.013} & \textbf{0.232$\pm$0.009} & \textbf{0.046$\pm$0.012} & \textbf{0.023$\pm$0.003} & \textbf{0.209$\pm$0.002} \\
\cmidrule(l){1-10}

\multirow{5}{*}{\rotatebox{90}{\scriptsize FUK+KYO+NGO$\to$OSA}} & w/o graph encoder & 0.398$\pm$0.025 & 0.245$\pm$0.015 & 0.255$\pm$0.051 & 0.134$\pm$0.018 & 0.284$\pm$0.018 & 0.064$\pm$0.002 & 0.030$\pm$0.001 & 0.222$\pm$0.001 \\
 & w/o flow-matching & 0.379$\pm$0.032 & 0.242$\pm$0.014 & 0.232$\pm$0.029 & 0.127$\pm$0.015 & 0.264$\pm$0.006 & 0.092$\pm$0.001 & 0.037$\pm$0.001 & 0.228$\pm$0.001 \\
 & w/o copy-gate & 0.483$\pm$0.013 & 0.257$\pm$0.006 & 0.306$\pm$0.079 & 0.140$\pm$0.072 & 0.267$\pm$0.007 & 0.054$\pm$0.002 & 0.025$\pm$0.000 & 0.205$\pm$0.002 \\
 & w/o AR-mask & 0.422$\pm$0.019 & 0.249$\pm$0.012 & 0.214$\pm$0.037 & 0.117$\pm$0.010 & 0.291$\pm$0.002 & 0.102$\pm$0.000 & 0.044$\pm$0.000 & 0.262$\pm$0.003 \\
 & w/o condition & 0.456$\pm$0.067 & 0.270$\pm$0.063 & 0.197$\pm$0.012 & 0.114$\pm$0.005 & 0.270$\pm$0.004 & 0.060$\pm$0.001 & 0.030$\pm$0.001 & 0.249$\pm$0.002 \\
 & \textbf{\method{}} & \textbf{0.367$\pm$0.026} & \textbf{0.150$\pm$0.017} & \textbf{0.192$\pm$0.054} & \textbf{0.053$\pm$0.035} & \textbf{0.259$\pm$0.006} & \textbf{0.050$\pm$0.002} & \textbf{0.023$\pm$0.001} & \textbf{0.204$\pm$0.002} \\
\cmidrule(l){1-10}

\multirow{5}{*}{\rotatebox{90}{\scriptsize ANK+BUR+IZM$\to$IST}} & w/o graph encoder & 0.302$\pm$0.048 & 0.160$\pm$0.022 & 0.187$\pm$0.038 & 0.050$\pm$0.018 & 0.613$\pm$0.004 & 0.055$\pm$0.000 & 0.015$\pm$0.001 & 0.264$\pm$0.001 \\
 & w/o flow-matching & 0.275$\pm$0.019 & 0.098$\pm$0.009 & 0.129$\pm$0.037 & 0.127$\pm$0.013 & 0.653$\pm$0.018 & 0.108$\pm$0.002 & 0.029$\pm$0.000 & 0.226$\pm$0.001 \\
 & w/o copy-gate & 0.332$\pm$0.028 & 0.195$\pm$0.024 & 0.252$\pm$0.123 & 0.121$\pm$0.091 & 0.463$\pm$0.002 & 0.046$\pm$0.001 & 0.012$\pm$0.000 & 0.217$\pm$0.002 \\
 & w/o AR-mask & 0.536$\pm$0.059 & 0.268$\pm$0.073 & 0.595$\pm$0.092 & 0.446$\pm$0.104 & 0.514$\pm$0.001 & 0.083$\pm$0.001 & 0.019$\pm$0.000 & 0.446$\pm$0.001 \\
 & w/o condition & 0.275$\pm$0.009 & 0.091$\pm$0.010 & 0.159$\pm$0.027 & 0.106$\pm$0.005 & 0.424$\pm$0.006 & 0.062$\pm$0.002 & 0.011$\pm$0.000 & 0.229$\pm$0.002 \\
 & \textbf{\method{}} & \textbf{0.195$\pm$0.047} & \textbf{0.056$\pm$0.023} & \textbf{0.117$\pm$0.007} & \textbf{0.023$\pm$0.005} & \textbf{0.399$\pm$0.005} & \textbf{0.034$\pm$0.001} & \textbf{0.010$\pm$0.013} & \textbf{0.209$\pm$0.001} \\
\bottomrule
\end{tabular}
\end{table*}

\subsubsection{Baselines.}
\label{sec:baselines}
We compare our method with six adaptation baselines. \textbf{Markov} \cite{gambs2012next} is a category-level Markov chain that is estimated on the source and grounded to the nearest target POIs. \textbf{DITRAS} \cite{pappalardo2018data} is a diary-based generator with category-level transitions and a gravity-style grounding that places each step by target POI density and distance. \textbf{Gravity} \cite{barbosa2018human} is the classic spatial-interaction model whose distance-decay exponents are fit on the source and applied to the target. \textbf{MIRAGE} \cite{deng2025revisiting} is a graph-enhanced discrete diffusion model that we adapt to the target city by mapping the output representations to the nearest POIs in the target city. \textbf{COLA} \cite{wang2024cola} is a cross-city Mobility Transformer whose output vocabulary is extended to the target POIs, and \textbf{GTG} \cite{wang2025gtg} is a graph Transformer with domain adaptation capability. Here the statistics from source cities are as its input to adapt our zero-shot setting. 

\subsubsection{Evaluation Metrics.}
\label{sec:metrics}
Following standard practice \cite{deng2025revisiting, deng2025marionette}, we evaluate the generation results in two aspects. One aspect is the \emph{downstream task utility} based on “training on synthetic data, testing on real data”. We train prediction and recommendation models using the generated trajectories, test them on a real target set, and report the gap between these models and those trained on real data. The prediction models include FPMC \cite{rendle2010factorizing}, BERT4Rec \cite{sun2019bert4rec}, SRGNN \cite{wu2019session}, and SASRec \cite{kang2018self}. Their gaps are averaged into \textbf{Pred-MAPE} and \textbf{Pred-MSPE}. The recommendation models include DMF \cite{xue2017deep}, MultiVAE \cite{liang2018variational}, NeuMF \cite{he2017neural}, and BPR \cite{rendle2012bpr}. Their gaps are averaged into \textbf{Rec-MAPE} and \textbf{Rec-MSPE}. 
The other aspect is \emph{distributional similarity} between generated and real trajectories. \textbf{OD-Sim} is the similarity of the origin-destination flow. It reflects where trips start and end. Three metrics are Jensen-Shannon divergences between the generated and the real distributions. \textbf{Distance} compares the lengths of moves. \textbf{Speed} compares how fast people move. \textbf{Time-Cat} compares which category of place is visited at which hour.

\subsection{Main Results}
\label{sec:main}

The comparisons of different methods on the four transfer settings are shown in Table~\ref{tab:main}. \textbf{\method{}} achieves the best results on both downstream utility and distributional fidelity in all four settings. Across all 14 transfer settings, ten of which are reported in Appendix~\ref{app:full}, \textbf{\method{}} lowers the average error over the best baseline of each metric by about 15\% on distributional fidelity and about 3\% on downstream utility. Specifically, the average gain over the eight metrics reaches more than 33\% on the \{KYO+NGO$\to$OSA\} transfer. Among the baselines, the neural methods including \textbf{MIRAGE}, \textbf{GTG} and \textbf{COLA} tend to perform better on the utility metrics, while the classical models \textbf{Markov}, \textbf{DITRAS} and \textbf{Gravity} are stronger on the distributional metrics. The former can model complex spatiotemporal distributions and learn sequence-level semantics, which benefits downstream tasks. The latter are built directly on trajectory statistics, so they capture the overall trends more easily. \textbf{\method{}} contains both ingredients, pairing a neural generator that models transitions at the sequence level with a behavior graph that aggregates their statistics over the target POIs. In addition, comparing the two Japanese transfers in Table~\ref{tab:main} and the Turkish settings in Appendix~\ref{app:full}, moving from a single source city to more source cities changes the results only marginally. More source trajectories bring limited gains to target-trajectory-free generation. This suggests that the knowledge \textbf{\method{}} models is shared across cities, and a small set of source cities is enough to learn it.

\begin{figure*}[htbp]
\centering
\subfigtopskip=-1pt 
\subfigcapskip=-5pt 
\subfigure[$\textbf{NYC} \rightarrow \textbf{TKY}$]{
	\label{l_3}
	\includegraphics[width=0.24\linewidth]{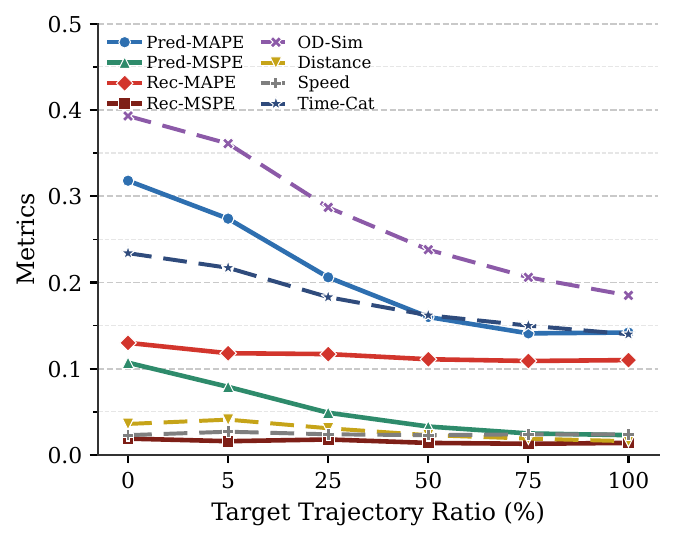}}
\subfigure[$\textbf{KYO+NGO} \rightarrow \textbf{OSA}$]{
	\label{l_3}
	\includegraphics[width=0.24\linewidth]{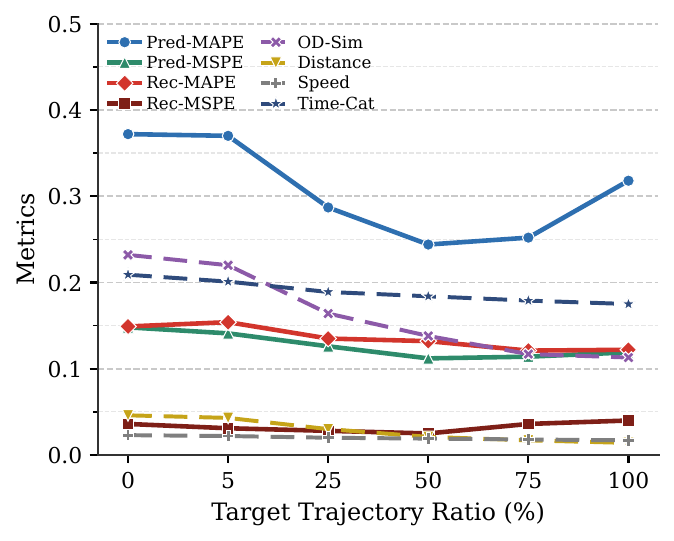}}
\subfigure[$\textbf{FUK+KYO+NGO} \rightarrow \textbf{OSA}$]{
	\label{l_1} 
	\includegraphics[width=0.24\linewidth]{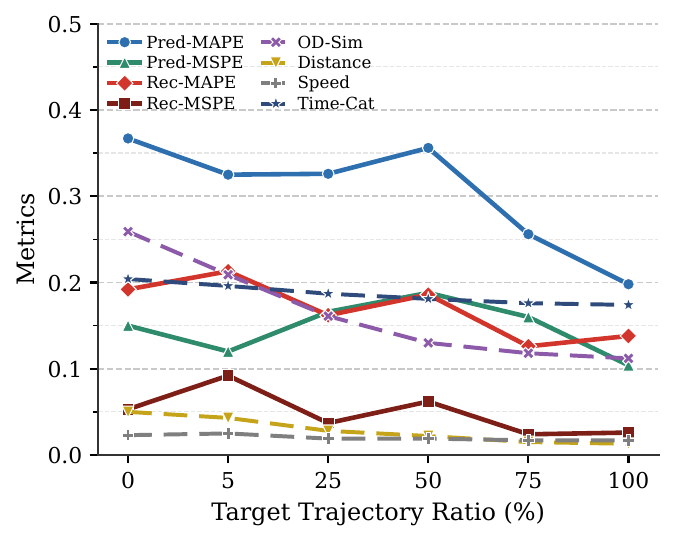}}
\subfigure[$\textbf{ANK+BUR+IZM} \rightarrow \textbf{IST}$]{
	\label{l_1} 
	\includegraphics[width=0.24\linewidth]{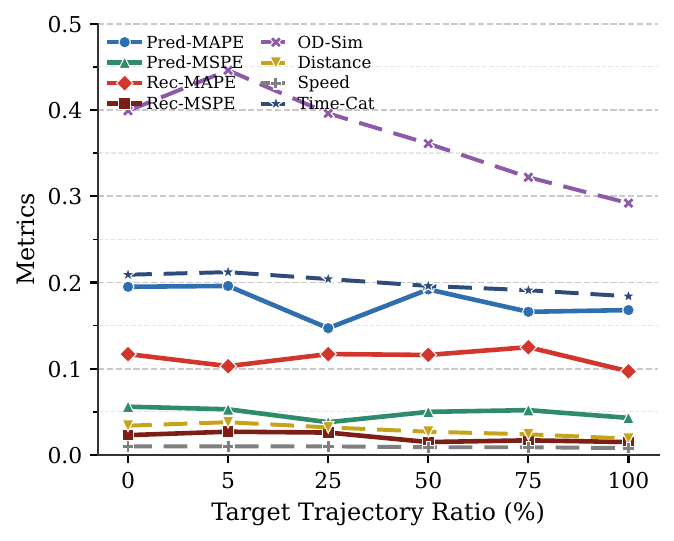}}
\vskip -0.1in
\caption{Impact of the target trajectory ratio for few-shot adapting on four transfers.}
\label{fig:fewshot} 
\end{figure*}

\subsection{Ablation Study}
\subsubsection{Ablation of Different Modules}
\label{sec:ablation}
To evaluate the efficacy of \textbf{\method{}}'s components, we conduct an ablation study comparing five variants (details in Appendix \ref{variants}). Table~\ref{tab:ablation} reports results on four transfers, demonstrating that deleting any module leads to performance declines, underscoring their critical contributions. Removing flow-matching hurts distributional metrics most, because the basic regressor keeps only mean movement and inflates Distance divergence. Conversely, removing the copy-gate barely moves distributions but raises Recommendation error, since walks without revisits lose the repeat structure recommenders learn. Removing the auto-regressive mask breaks step order and blows up prediction error. The graph encoder and history-condition underpin both aspects. Dropping either raises error across the board.

\subsubsection{Ablation of Data Efficiency}
\label{sec:fewshot}

We test how the error changes as an increasing fraction of the target trajectories is used for adaptation. We concatenate the source-city trajectories with a growing fraction of the target-city trajectories and train \textbf{\method{}} on the mixture. This moves the model from the target-trajectory-free setting to cross-city generation as more target trajectories are added. Figure~\ref{fig:fewshot} reports the eight metrics on four transfers. 
First, the distributional errors decrease steadily as more target trajectories are used, because the statistics of the real trajectories can be collected and learned directly. Downstream utility errors fluctuate, yet the overall trend remains downward. Second, across these transfer settings, only a small fraction of the target data brings the error close to the in-domain lower bound. This shows the transferred movement pattern covers most knowledge these tasks need. Third, the Metropolis setting behaves differently. Its POI vocabulary is larger, so its utility error keeps decreasing as more target data is added. This shows in-domain trajectories still help at larger data scales.

\begin{figure}[htbp]
\centering
\vspace{-0.1in}
\subfigtopskip=-1pt 
\subfigure[Absolute Destinations]{
	\label{l_1} 
	\includegraphics[width=0.45\linewidth]{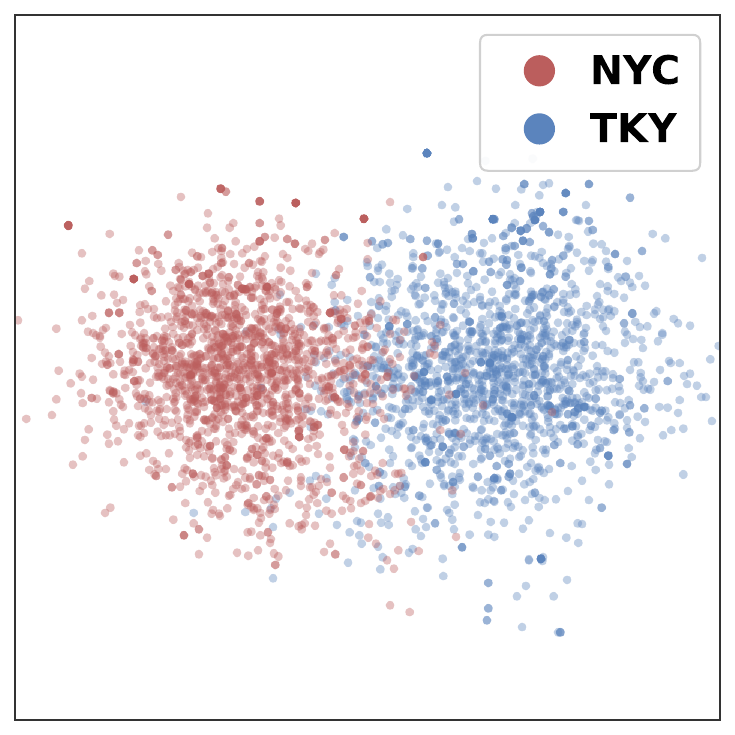}}
\subfigure[Relative Transitions]{
	\label{l_2}
	\includegraphics[width=0.45\linewidth]{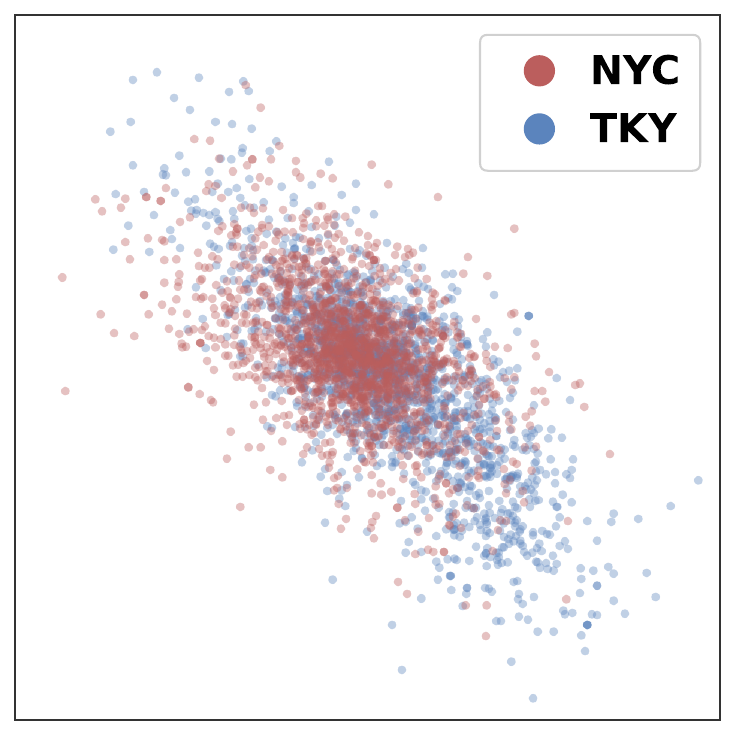}}
\caption{The embedding visualization of destinations and transitions on NYC and TKY.}
\vspace{-0.2in}
\label{fig:matched-transfer} 
\end{figure}

\begin{figure}[ht]
\vskip -0.2in
\begin{center}
\centerline{\includegraphics[width=\linewidth]{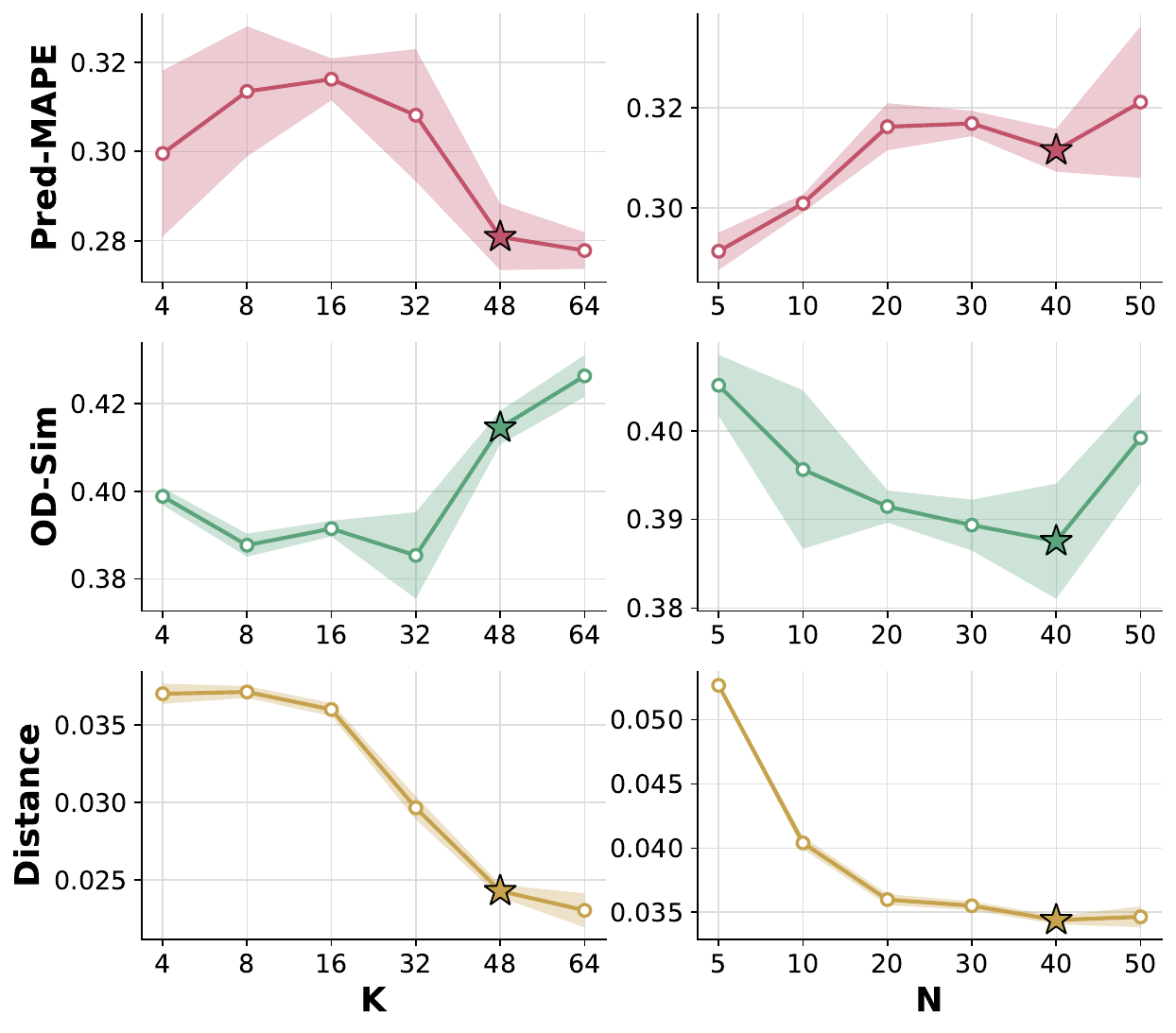}}
\vspace{-5pt}
\caption{Impact of $K$ and $N$ for grounding on NYC $\to$ TKY.
}
\label{fig:kn}
\end{center}
\vskip -0.2in
\end{figure}

\subsection{Analysis and Visualization}
\label{sec:analysis}

\subsubsection{Visualization of Cross-city Transitions and Destinations}

We visually examine the premise behind our design, that a relative transition is more
comparable across cities than an absolute destination. Using the frozen shared encoder we
embed the POIs of NYC and TKY. Then we read each check-in transition $a\!\to\!b$ and obtain the destination $h_b$ and the transition $\boldsymbol{\delta}=h_b-h_a$. We compare only like-for-like movements under a matched context. Each transition is grouped by its time-of-day bin and by the category of its destination, only groups populated by both cities are kept, and the larger city is subsampled so the two cities contribute equally.
Figure~\ref{fig:matched-transfer} shows the matched samples of the two cities for each
representation. For the absolute destinations, the two cities form clearly separated
clusters, so they still reveals which city a movement belongs to. For the relative transitions, the two cities overlap almost entirely, demonstrating that it is nearly indistinguishable across cities. Because functionally equivalent POIs in the two cities differ by an approximately constant translation, this translation is apparent in the absolute destination but cancels in the displacement $\boldsymbol{\delta}$, consistent with Assumption~\ref{asm:gauge} and Theorem~\ref{thm:delta}.

\subsubsection{Analysis of Sampling Budgets}

We vary the two sampling budgets of \method{}, the number of flow steps $N$ in the generator and the number of transition samples $K$ that build the behavior graph, and read their effect on one utility metric and two distributional metrics. Figure~\ref{fig:kn} reports the curves. 
As shown in Figure~\ref{fig:kn}, distributional fidelity is significantly affected by $N$ and $K$. On the one hand, more flow steps integrate the transition dynamics more accurately, so the Distance divergence drops sharply at $N{=}5$ to $N{=}20$ and stays flat afterward, and OD-Sim reaches its lowest value near $N{=}40$. The utility metric barely moves with $N$. A small number of steps is therefore enough for good fidelity. 
On the other hand, an increase in $K$ results in a denser behavior graph, leading to more accurate downstream prediction tasks. At the same time, however, the origin-destination structure becomes more diverse, causing a slight degradation in the OD-Sim metric.
Larger $K$ thus helps prediction and fine distance at a small cost to the OD match.

\section{Conclusion}

This paper investigates the problem of target-trajectory-free generation for those cities without accessible trajectories. The primary challenge lies in transferring movement knowledge across cities that share no POIs. To tackle this challenge, we propose \method{}, a transfer-and-ground framework. Our approach models relative transitions rather than absolute destinations with a flow-matching prior, and grounds them onto the target POIs with a behavior-graph and a pointer network.
Extensive experiments across fourteen transfers confirm that \method{} outperforms adaptation baselines in distributional fidelity and downstream utility. Controlled analyses further show that relative transitions contain less city-specific information under matched contexts.







\bibliographystyle{ACM-Reference-Format}
\bibliography{sample-base}

\clearpage
\newpage

\appendix

\section{Proof}
\subsection{Proof of Theorem~\ref{thm:delta}}
\label{app:proof_d}
\begin{proof}
For transferability, Assumption~\ref{asm:gauge} gives:
$\mathbf{h}_{\varphi(a)}=\mathbf{h}_{a}+\mathbf{t}_{AB}$ and $\mathbf{h}_{\varphi(b)}=\mathbf{h}_{b}+\mathbf{t}_{AB}$, hence
$\boldsymbol{\delta}(\varphi(a),\varphi(b))
=(\mathbf{h}_{b}+\mathbf{t}_{AB})-(\mathbf{h}_{a}+\mathbf{t}_{AB})
=\boldsymbol{\delta}(a,b)$,
so the offset cancels and Eq.~\eqref{eq:transferable} holds. For completeness, let $F$ be transferable. Since the offset between two cities is arbitrary, transferability amounts to $F(\mathbf{h}_{a}+\mathbf{t},\mathbf{h}_{b}+\mathbf{t})=F(\mathbf{h}_{a},\mathbf{h}_{b})$ for every $\mathbf{t}\in\mathbb{R}^{d}$. Choosing $\mathbf{t}=-\mathbf{h}_{a}$ yields $F(\mathbf{h}_{a},\mathbf{h}_{b})=F(\mathbf{0},\mathbf{h}_{b}-\mathbf{h}_{a})$, so $F=\tilde{F}(\boldsymbol{\delta})$ with $\tilde{F}(\mathbf{u})=F(\mathbf{0},\mathbf{u})$, which is the classical factorization of an invariant function through a maximal invariant \cite{lehmann2005testing}. Finally, functionally corresponding moves take identical values of $\boldsymbol{\delta}$, so a distribution fitted on source moves describes the corresponding moves of any city satisfying the assumption. The same argument applies verbatim to the time gap and the ground displacement, which are differences as well.
\end{proof}

\subsection{Corollary and the Proof}
\label{app:proof_c}

\begin{corollary}[Context Conditioning]
The displacement is antisymmetric, $\boldsymbol{\delta}(b,a)=-\boldsymbol{\delta}(a,b)$, and telescopes to zero along any closed chain of moves. Hence a distribution over $\boldsymbol{\delta}$ alone carries no preference about where a move starts, which direction it takes, or whether it returns to an earlier place. Let $\mathbf{c}$ be a context variable that is itself transferable in the sense of Definition~\ref{def:transferable}. Then the conditional distribution $p(\boldsymbol{\delta}\mid\mathbf{c})$ remains transferable while expressing such anchor and history dependent preferences.
\label{cor:context}
\end{corollary}

Theorem~\ref{thm:delta} and Corollary~\ref{cor:context} dictate the design of our method. The encoder should earn the Assumption by aligning functionally equivalent POIs up to a shift Sec.~\ref{sec:encoder}. The generator should be placed on $\boldsymbol{\delta}$ and, by the Corollary, conditioned on the trajectory context Sec.~\ref{sec:fm}. Returning to familiar places is likewise a history dependent preference that $\boldsymbol{\delta}$ alone cannot carry, so revisiting is decided by a dedicated module conditioned on the walk history. Then, trajectories are rebuilt by anchoring at a target POI and composing generated displacements, which is described in Sec.~\ref{sec:generation}.

\begin{proof}
Antisymmetry is immediate, and for any chain $p_{1},\dots,p_{K}$ the sum $\sum_{k=1}^{K-1}\boldsymbol{\delta}(p_{k},p_{k+1})=\mathbf{h}_{p_{K}}-\mathbf{h}_{p_{1}}$ vanishes when $p_{K}=p_{1}$. A distribution over $\boldsymbol{\delta}$ alone assigns equal likelihood to moves with equal difference regardless of their anchor, and determines the reverse move up to sign, so no start, direction, or return preference can be encoded. For the conditional law, if $\mathbf{c}$ is transferable then the pair $(\boldsymbol{\delta},\mathbf{c})$ takes identical values on functionally corresponding moves of two cities satisfying Assumption~\ref{asm:gauge}, hence the conditional distribution $p(\boldsymbol{\delta}\mid\mathbf{c})$ estimated on source cities coincides with that of the corresponding target moves, while depending on the anchor and the history through $\mathbf{c}$.
\end{proof}

\section{Related Work}

\subsection{Human Trajectory Generation}

Most trajectory generators are trained and deployed within a single city \cite{kapp2023generative}. Classical methods synthesize movement with activity models \cite{yuan2022activity,yuan2023learning}, adversarial or imitation learning \cite{feng2020learning,choi2021trajgail,jiang2023continuous,rao2020lstm}, and privacy-preserving generators \cite{wang2023pategail,long2023practical}. 
More recent work turns to diffusion models. Some denoise GPS-level trajectories, optionally conditioned on the road network \cite{zhu2023difftraj,zhu2024controltraj,wei2024diff,rao2025seed}. Others generate check-in sequences or refine them from coarse to fine \cite{deng2025revisiting,guo2026leveraging}. Beyond direct generation, high-fidelity synthetic-data pipelines scale these methods up \cite{zhu2023synmob}.
A parallel line uses large language models and agents \cite{park2023generative} to simulate mobility \cite{shao2026generating,li2024geo,li2024more}, mine trajectories \cite{liu2026traj}, and recommend next POIs by reasoning \cite{zhuang2026think2go,li2024large}.
Whether statistical, diffusion, or LLM-based, these methods are trained and generate within a single city, relying on its historical trajectories, thus hard to generalize to cities without trajectories.

\subsection{Cross-City Transfer for Mobility}

Cross-city transfer methods reuse mobility knowledge from data-rich cities. Most effort targets spatiotemporal prediction. Meta learning and region re-weighting adapt a source model to a target city \cite{pan2019urban,jin2022selective,wang2018cross,finn2017model}. Cross-city representation learning aligns regions or knowledge-graph embeddings for downstream reuse \cite{zhang2023deep,chen2025cross}. These methods act over a shared spatial grid rather than a discrete POI vocabulary, and still fit part of the model on target-city observations. Cross-city trajectory generation has been attempted recently, but the closest methods such as COLA, GTG, and CRAFT \cite{wang2024cola,wang2025gtg,li2026cross} remain bound to a target-specific location space and still require trajectories from the target city. Generating individual check-in trajectories for a city without trajectories remains challenging.

\section{Implementation Details}
\subsection{Datasets Details}

\label{sec:datasets}

We evaluate cross-city zero-shot generation on public Foursquare check-in data \cite{yang2014modeling, yang2016participatory} covering ten cities in three regions. New York City (NYC) and Tokyo (TKY) are large metropolises. The other eight cities come in two clusters: a Japanese cluster of Fukuoka (FUK), Kyoto (KYO), Nagoya (NGO), and Osaka (OSA); a Turkish cluster of Ankara (ANK), Bursa (BUR), Istanbul (IST), and Izmir (IZM). In every city a trajectory is the time-ordered sequence of one user's check-ins within one local day, truncated at fifteen steps. For NYC and TKY we keep days with at least three check-ins, whereas for the two clusters we keep denser days with at least six. Table~\ref{tab:datastats} reports the per-city statistics. From these cities we construct $14$ source-to-target transfers. Sec. \ref{sec:main} presents four pairs, covering both single-source and multi-source settings. Due to space limitations, the results on the remaining ten pairs are reported in Appendix~\ref{app:full}.

\begin{table}[t]
\centering
\small
\caption{Per-city statistics of the ten Foursquare cities.}
\label{tab:datastats}
\begin{tabular}{llccccc}
\toprule
Region & City & \#POI & \#Cat. & \#Traj. & \#Check-in & Avg.\,len. \\
\midrule
\multirow{2}{*}{Metropolis} & NYC & 38,333 & 250 & 27,662 & 135,717 & 4.9 \\
 & TKY & 61,858 & 247 & 73,999 & 397,416 & 5.4 \\
\midrule
\multirow{4}{*}{Japan} & FUK & 3,673 & 256 & 2,577 & 22,294 & 8.7 \\
 & KYO & 5,124 & 265 & 3,546 & 30,967 & 8.7 \\
 & NGO & 5,796 & 286 & 4,321 & 36,147 & 8.4 \\
 & OSA & 11,974 & 344 & 10,038 & 87,257 & 8.7 \\
\midrule
\multirow{4}{*}{Turkey} & ANK & 10,853 & 374 & 5,782 & 43,503 & 7.5 \\
 & BUR & 3,862 & 303 & 1,962 & 14,654 & 7.5 \\
 & IST & 19,996 & 388 & 9,705 & 73,040 & 7.5 \\
 & IZM & 10,826 & 366 & 6,513 & 48,702 & 7.5 \\
\bottomrule
\end{tabular}
\end{table}

\subsection{Ablation Variants} \label{variants}
We design five variants to analyze the contribution of components:
\begin{enumerate}[0]
\item[$\bullet$] \textbf{w/o graph encoder} replaces the relation-aware GATv2 of Eq. \eqref{eq:gatv2} and the aggregation of Eq. \eqref{eq:enclayer} with a single linear projection of the raw features, $\mathbf{h}_i=\mathbf{W}[\,\bar{\boldsymbol{\ell}}_i\,\|\,\mathbf{s}_{c_i}\,]$, dropping all message passing. It tests whether cross-city transfer needs graph propagation which can aggregate the geographic information.
\item[$\bullet$] \textbf{w/o flow-matching} replaces the flow-matching objective of Eq. \eqref{eq:fmloss} with a deterministic MSE regressor that predicts $\boldsymbol{\delta}_k$ in one forward pass, and drops the Euler sampler of Eq. \eqref{eq:sample}. It predicts only the conditional mean of transition distribution.
\item[$\bullet$] \textbf{w/o copy-gate} fixes $\mathrm{prob}_k=0$ in Eq.~\eqref{eq:ptrheads}, removing the revisit branch so grounding becomes a pure behavior-graph walk. It tests the contribution of the revisit mechanism to our model.
\item[$\bullet$] \textbf{w/o AR-mask} removes the auto-regressive mask from the self-attention over the tokens $\mathbf{u}k$ of Eq.~\eqref{eq:token} and uses bidirectional attention. A future token carries $\mathbf{h}{p_k}$, the ground-truth that step $k$ should predict, so the mask blocks this leakage. It tests the role of the auto-regressive constraint in preventing next-step-leakage.
\item[$\bullet$] \textbf{w/o condition} zeros the conditioning in the velocity network of Eq. \eqref{eq:fmloss}, replacing $v_\theta(\boldsymbol{\delta}^{t},t,\mathbf{c}k)$ with $v\theta(\boldsymbol{\delta}^{t},t,\mathbf{0})$, so the flow samples the marginal $p(\boldsymbol{\delta})$ instead of $p(\boldsymbol{\delta}\mid\mathbf{c}_k)$. The current POI anchor is still injected at the snapping step of Eq. \eqref{eq:snap}, so the walk stays anchored. The variant is used to validating history conditioning's effect on sequence modeling.
\end{enumerate}

\section{Algorithm Analysis}\label{algorithm}

We analyze the computational complexity of \method{} stage by stage. Let $|\mathcal{P}|$ be the number of POIs of a city and $|\mathcal{E}|$ the number of edges of its POI graph, $d$ the shared embedding dimension, $L$ and $R$ the number of encoder layers and relations, $D$ and $B$ the width and depth of the flow-matching velocity network, $S$ its number of integration steps, $K$ the number of transitions sampled per POI to build the behavior graph, $\Delta$ the out-degree of a graph node, and $M$ and $\ell$ the number and the maximum length of the generated trajectories. \emph{(i) Cross-city encoding.} The relation-aware GATv2 passes messages along the edges and transforms the node features, which costs $O\big(L R (|\mathcal{E}|\,d + |\mathcal{P}|\,d^{2})\big)$. It is run once per city and is linear in the city size. \emph{(ii) Behavior-graph construction.} For each POI we condition the frozen generator on its embedding and draw $K$ transitions, and each draw integrates the velocity network for $S$ steps, which costs $O\big(|\mathcal{P}|\,K\,S\,B\,D^{2}\big)$; snapping every sample to its best-matching target POI adds $O\big(|\mathcal{P}|\,K\,d\big)$, with no long denoising chain because $S$ is a small constant. \emph{(iii) Grounding walk.} The copy-gate walk generates each step with one gate evaluation and one transition lookup, so producing $M$ trajectories of length up to $\ell$ costs $O\big(M \ell (d + \Delta)\big)$. Summing the three stages, the overall complexity is
\begin{equation*}
O\big(\, L R (|\mathcal{E}|\,d + |\mathcal{P}|\,d^{2}) \;+\; |\mathcal{P}|\,K\,S\,B\,D^{2} + |\mathcal{P}|\,K\,d \;+\; M \ell (d + \Delta) \,\big).
\end{equation*}
Every architecture term ($L, R, d, B, D, S, K, C, \Delta, \ell$) is a fixed hyperparameter that does not grow with the data, so \method{} scales linearly with the number of POIs $|\mathcal{P}|$ and the number of generated trajectories $M$, and it avoids the iterative denoising required by diffusion-based generators.
All three stages are linear in the city size $|\mathcal{P}|$ and the number of generated trajectories $M$. The cost is dominated by the behavior-graph construction, which draws $|\mathcal{P}|\,K$ transitions of $S$ integration steps each, while $S$ and $K$ are small constants. Because \method{} needs only a few integration steps and a single-pass grounding walk, it avoids the long iterative denoising of diffusion-based generators.

The algorithm is shown in Algorithm \ref{alg:method}. During the whole experiments, we use an Intel(R) Xeon(R) Gold 5318Y CPU and four NVIDIA GeForce RTX 4090 GPU with 24GB memory to train and perform inference.

\begin{algorithm}[]
\caption{Detailed Procedures of \method{}}
\label{alg:method}
\begin{algorithmic}[1]
\REQUIRE {Source city trajectories $\{\mathcal{T}_s\}_{s=1}^{S}$, Source city POIs $\mathcal{V}_s$. and target city POIs $\mathcal{V}_t$.}
\ENSURE {Generated trajectories $\hat{\mathcal{T}}_t$ in the target city.}

\textbf{// First Stage: Cross-City POI Alignment}

\STATE{Construct multi-relational POI graphs from source and target cities;}
\STATE{Initialize the Relation-aware GATv2;}
\WHILE{not converged}
\STATE{Compute relation-specific node representations 
$\{\mathbf{h}^{r}_i\}$ by GATv2 aggregation (Eq. \eqref{eq:gatv2});}
\STATE{Aggregate multiple relation representations into POI embeddings $\mathbf{h}_i$ (Eq. \eqref{eq:enclayer});}

\STATE{Update encoder parameters by minimizing alignment objective  $\mathcal{L}_{\mathrm{align}}$ (Eq. \eqref{eq:encloss});}
\ENDWHILE

\textbf{// Second Stage: Flow Matching in Transition Space}

\STATE{Construct transition vectors 
$\boldsymbol{\delta}_k$ from source trajectories 
(Eq. \eqref{eq:delta});}
\STATE{Initialize the velocity network $v_\theta$;}
\WHILE{not converged}
\STATE{Sample noise transition $\boldsymbol{\delta}^{0}\sim\mathcal{N}(\mathbf{0},\mathbf{I})$ and real transition $\boldsymbol{\delta}^{1}$;}
\STATE{Interpolate transition samples $\boldsymbol{\delta}^{t}$ along the flow path;}
\STATE{Construct trajectory tokens 
$\mathbf{u}_k$ (Eq. \eqref{eq:token});}
\STATE{Predict transport velocity 
$v_\theta(\boldsymbol{\delta}^{t},t,\mathbf{c}_k)$;}
\STATE{Update $v_\theta$ by minimizing flow matching loss 
$\mathcal{L}_{\mathrm{fm}}$ (Eq. \eqref{eq:fmloss});}
\ENDWHILE

\textbf{// Third Stage: Target-City Grounding}

\FOR{each target POI $i\in\mathcal{V}_t$}
\STATE{Generate transition samples $\{\hat{\boldsymbol{\delta}}^{(s)}\}_{s=1}^{K}$ conditioned on $\mathbf{h}_i$;}
\STATE{Ground generated transitions to target POIs using the matching function (Eq. \eqref{eq:snap});}
\STATE{Calculate transition probability $P(i\rightarrow j)$ (Eq. \eqref{eq:edgeprob});}
\ENDFOR
\STATE{Construct target behavior graph $\mathcal{G}_B$;}
\STATE{Initialize Pointer Network with source trajectories;}
\WHILE{not converged}
\STATE{Encode current POI and contextual features 
$\boldsymbol{\xi}_k$ (Eq. \eqref{eq:ptrfeat});}
\STATE{Predict return probability and revisit distribution using Pointer Network heads (Eq. \eqref{eq:ptrheads});}
\STATE{Update Pointer Network parameters using return gate loss and revisit prediction loss;}
\ENDWHILE

\textbf{// Trajectory Generation}

\STATE{Sample initial POI and timestamp from target city;}
\WHILE{trajectory length is not reached}
\STATE{Generate transition by solving flow ODE with Euler sampling (Eq. \eqref{eq:sample});}
\STATE{Decide return or exploration using Pointer Network;}
\STATE{Sample next POI from revisit distribution or behavior graph $\mathcal{G}_B$;}
\ENDWHILE
\STATE{\bfseries return} {Obtain generated target-city trajectories $\hat{\mathcal{T}}_t$.}
\end{algorithmic}
\end{algorithm}


Besides, we further provide the running time comparison between \method{} and SOTA baselines in training and generating phases in Tables \ref{tab:efficiency} to verify effectiveness of the proposed method.
\begin{table}[t]
\centering
\SMALL
\caption{Training time per epoch (summed over all training stages) and generation time
for 10k trajectories.}
\label{tab:efficiency}
\begin{tabular}{l r rrr c}
\toprule
Transfer & MIRAGE (s) & GTG (s) & COLA (s) & \textbf{Nomad} (s) \\
\midrule
TKY$\to$NYC          & 27.3 & 446.0 & 68.0 & 9.6 \\
NYC$\to$TKY          & 8.9  & 277.0 & 25.8 & 5.4 \\
OSA$\to$NGO          & 3.2  & 86.0  & 9.1  & 3.9 \\
OSA$\to$KYO          & 3.2  & 86.0  & 9.1  & 3.9 \\
OSA$\to$FUK          & 3.2  & 86.0  & 9.1  & 3.9 \\
NGO$\to$OSA          & 1.6  & 41.8 & 4.5  & 3.4 \\
KYO+NGO$\to$OSA      & 2.9  & 79.0  & 11.7 & 3.7 \\
FUK+KYO+NGO$\to$OSA  & 2.7  & 105.0 & 10.3 & 3.9 \\
IST$\to$IZM          & 3.2  & 144.0 & 8.9  & 3.8 \\
IST$\to$ANK          & 3.2  & 144.0 & 8.9  & 3.8 \\
IST$\to$BUR          & 3.2  & 144.0 & 8.9  & 3.8 \\
IZM$\to$IST          & 2.8  & 78.0  & 6.2  & 3.6 \\
ANK+IZM$\to$IST      & 4.2  & 156.0 & 11.8 & 4.1 \\
ANK+BUR+IZM$\to$IST  & 4.9  & 184.0 & 12.1 & 4.2 \\
\midrule
Gen. 10k trajs & 241.7 & 17.0 & 53.0 & \bf13.3 \\
\bottomrule
\end{tabular}
\end{table}
From Table \ref{tab:efficiency}, MIRAGE is one to two seconds faster than our model On the small-scale datasets like Japanese and Turkish clusters. However, on the larger-scale NYC and TKY datasets, \textbf{\method{}} still completes a training epoch in under 10 seconds, whereas MIRAGE's speed drops significantly. 
GTG incurs substantially higher computational cost due to its complex sequence generation process, while MIRAGE and COLA remain slower despite their simpler architectures. The efficiency advantage becomes even more evident during inference. Generating 10k trajectories takes only 13.3 seconds for \textbf{\method{}}, compared with 17.0 seconds for GTG, 53.0 seconds for COLA, and 241.7 seconds for MIRAGE. This improvement mainly comes from flow matching module, which enables one-pass generation without the iterative denoising process required by diffusion-based generators. Moreover, grounding sampled transitions through the behavior graph involves only lightweight nearest-neighbor matching and probability lookup, introducing negligible additional overhead. These results demonstrate that \textbf{\method{}} is not only more accurate but also computationally efficient for large-scale trajectory synthesis.


\section{Supplementary Experiments}

\subsection{Full Results on All Transfers}
\label{app:full}

Tables~\ref{tab:app_metro}, \ref{tab:app_japan}, and \ref{tab:app_turkey} report the results on the ten transfer pairs not shown in the main text, under the same protocol and conventions as Table~\ref{tab:main}. The complete results further reveal that the advantage of \textbf{\method{}} is not driven by a small number of favorable transfer settings. It achieves a positive average gain over the strongest metric-wise baselines on seven of the ten additional pairs, including the geographically and structurally distinct TKY$\to$NYC setting, where the gain reaches 24.43\%. Its performance is particularly stable on the fidelity metrics: among these ten pairs, \textbf{\method{}} obtains the lowest Distance error in nine settings, the lowest Speed error in eight settings, and the lowest Time-Cat error in nine settings. These results show that the generated trajectories consistently preserve realistic travel ranges, movement speeds, and temporal-category patterns after being grounded to previously unseen POIs. The gains are not limited to distribution matching. Regarding TKY$\to$NYC and NGO$\to$OSA, \textbf{\method{}} ranks first across all four downstream metrics, yet on OSA$\to$NGO, OSA$\to$KYO, IST$\to$IZM, and IZM$\to$IST, it achieves the best or second-best result on most utility measures. 

Markov is competitive on prediction in several Japanese and Turkish pairs, COLA and GTG perform well on recommendation, and Gravity often provides a strong OD or Speed match. In contrast, \textbf{\method{}} more frequently maintains competitive utility and fidelity simultaneously rather than optimizing one aspect at the expense of the other. The three negative average gains, namely OSA$\to$FUK, IST$\to$IZM, and IST$\to$BUR, also clarify the remaining difficulty. In these settings, \textbf{\method{}} continues to achieve the best Distance, Speed, and Time-Cat results, but loses ground on one or more downstream or OD-based measures. Similarly, on IST$\to$BUR, MIRAGE performs best for prediction and Gravity for recommendation and OD-Sim, although \textbf{\method{}} still gives the closest distance, speed, and temporal-category distributions. Therefore, the negative aggregate values do not indicate uniformly poor transfer; rather, they arise from a mismatch between accurate marginal mobility statistics and the pair-specific transition or recommendation patterns required by certain downstream tasks. Finally, the comparison between trace IZM$\to$IST and trace ANK+IZM$\to$IST supports the observation in the main text: adding Ankara as an additional source changes most errors only slightly and does not alter the overall ranking pattern. Taken together, the full pairwise results show that \textbf{\method{}} transfers reliably across metropolitan, Japanese, and Turkish city groups, with its most consistent advantage lying in preserving the physical and temporal characteristics of mobility while remaining competitive for downstream use.

\begin{table*}[t]
\centering
\small
\caption{Full results on \textbf{Metropolis} transfer pair. Conventions as in Table~\ref{tab:main}.}
\label{tab:app_metro}
\begin{tabular}{ll cccc cccc c}
\toprule
& \multirow{2}{*}{Method} & \multicolumn{4}{c}{Downstream Utility} & \multicolumn{4}{c}{Distributional Fidelity}  &  \\
\cmidrule(lr){3-6}\cmidrule(lr){7-10}
& & Pred-MAPE\,$\downarrow$ & Pred-MSPE\,$\downarrow$ & Rec-MAPE\,$\downarrow$ & Rec-MSPE\,$\downarrow$ & OD-Sim\,$\downarrow$ & Distance\,$\downarrow$ & Speed\,$\downarrow$ & Time-Cat\,$\downarrow$  & \\
\midrule
\multirow{7}{*}{\rotatebox{90}{TKY$\to$NYC}} & Markov \cite{gambs2012next} & 0.619$\pm$0.071 & 0.400$\pm$0.094 & 0.449$\pm$0.076 & 0.229$\pm$0.071 & 0.585$\pm$0.005 & 0.181$\pm$0.001 & -- & --  & \multirow{7}{*}{\rotatebox{90}{\textcolor{red}{+24.43\%}}} \\
 & DITRAS \cite{pappalardo2018data} & 0.776$\pm$0.083 & 0.641$\pm$0.122 & 0.696$\pm$0.092 & 0.547$\pm$0.121 & 0.473$\pm$0.008 & 0.143$\pm$0.000 & 0.040$\pm$0.000 & 0.284$\pm$0.002  & \\
 & Gravity \cite{barbosa2018human} & 0.647$\pm$0.068 & 0.462$\pm$0.103 & 0.391$\pm$0.062 & 0.183$\pm$0.067 & \uline{0.312$\pm$0.002} & 0.128$\pm$0.002 & \uline{0.029$\pm$0.000} & 0.236$\pm$0.001  & \\
 & MIRAGE \cite{deng2025revisiting} & 0.608$\pm$0.033 & 0.418$\pm$0.042 & 0.445$\pm$0.054 & 0.290$\pm$0.058 & 0.946$\pm$0.038 & \uline{0.083$\pm$0.001} & 0.063$\pm$0.001 & 0.270$\pm$0.002  & \\
 & GTG \cite{wang2025gtg} & \uline{0.575$\pm$0.004} & \uline{0.328$\pm$0.004} & \uline{0.285$\pm$0.037} & \uline{0.088$\pm$0.024} & 0.710$\pm$0.025 & 0.235$\pm$0.002 & 0.121$\pm$0.001 & 0.223$\pm$0.002  & \\
 & COLA \cite{wang2024cola} & 0.762$\pm$0.053 & 0.634$\pm$0.082 & 0.406$\pm$0.033 & 0.184$\pm$0.033 & 0.522$\pm$0.006 & 0.254$\pm$0.002 & 0.212$\pm$0.001 & \uline{0.161$\pm$0.001}  & \\
 & \textbf{\method{}} & \textbf{0.501$\pm$0.005} & \textbf{0.254$\pm$0.005} & \textbf{0.258$\pm$0.011} & \textbf{0.080$\pm$0.015} & \textbf{0.287$\pm$0.003} & \textbf{0.013$\pm$0.001} & \textbf{0.018$\pm$0.002} & \textbf{0.143$\pm$0.001}  & \\
\bottomrule
\end{tabular}
\end{table*}

\begin{table*}[t]
\centering
\small
\caption{Full results on \textbf{Japanese-cluster} transfer pairs. Conventions as in Table~\ref{tab:main}.}
\label{tab:app_japan}
\begin{tabular}{ll cccc cccc c}
\toprule
& \multirow{2}{*}{Method} & \multicolumn{4}{c}{Downstream Utility} & \multicolumn{4}{c}{Distributional Fidelity}  &  \\
\cmidrule(lr){3-6}\cmidrule(lr){7-10}
& & Pred-MAPE\,$\downarrow$ & Pred-MSPE\,$\downarrow$ & Rec-MAPE\,$\downarrow$ & Rec-MSPE\,$\downarrow$ & OD-Sim\,$\downarrow$ & Distance\,$\downarrow$ & Speed\,$\downarrow$ & Time-Cat\,$\downarrow$  & \\
\midrule
\multirow{7}{*}{\rotatebox{90}{OSA$\to$NGO}} & Markov \cite{gambs2012next} & \textbf{0.365$\pm$0.028} & 0.172$\pm$0.020 & \uline{0.233$\pm$0.091} & \textbf{0.103$\pm$0.069} & 0.399$\pm$0.024 & 0.199$\pm$0.003 & -- & --  & \multirow{7}{*}{\rotatebox{90}{\textcolor{red}{+6.24\%}}} \\
 & DITRAS \cite{pappalardo2018data} & \uline{0.382$\pm$0.039} & 0.180$\pm$0.032 & 0.287$\pm$0.108 & 0.155$\pm$0.114 & \uline{0.312$\pm$0.013} & 0.159$\pm$0.003 & 0.058$\pm$0.003 & 0.344$\pm$0.002  & \\
 & Gravity \cite{barbosa2018human} & 0.463$\pm$0.113 & 0.253$\pm$0.113 & 0.449$\pm$0.135 & 0.398$\pm$0.160 & 0.350$\pm$0.013 & 0.182$\pm$0.002 & \uline{0.024$\pm$0.001} & 0.367$\pm$0.001  & \\
 & MIRAGE \cite{deng2025revisiting} & 0.455$\pm$0.030 & 0.233$\pm$0.021 & 0.335$\pm$0.094 & 0.224$\pm$0.153 & 0.837$\pm$0.145 & \textbf{0.056$\pm$0.003} & 0.034$\pm$0.003 & 0.405$\pm$0.008  & \\
 & GTG \cite{wang2025gtg} & 0.398$\pm$0.039 & \uline{0.171$\pm$0.029} & 0.365$\pm$0.020 & 0.395$\pm$0.010 & 0.736$\pm$0.030 & 0.153$\pm$0.003 & 0.165$\pm$0.002 & \uline{0.317$\pm$0.007}  & \\
 & COLA \cite{wang2024cola} & 0.401$\pm$0.040 & 0.186$\pm$0.034 & 0.241$\pm$0.137 & 0.203$\pm$0.205 & 0.477$\pm$0.019 & 0.243$\pm$0.005 & 0.145$\pm$0.004 & 0.332$\pm$0.001  & \\
 & \textbf{\method{}} & 0.382$\pm$0.031 & \textbf{0.167$\pm$0.022} & \textbf{0.170$\pm$0.123} & \uline{0.103$\pm$0.171} & \textbf{0.292$\pm$0.015} & \uline{0.059$\pm$0.002} & \textbf{0.021$\pm$0.000} & \textbf{0.280$\pm$0.003}  & \\
\cmidrule(l){1-11}
\multirow{7}{*}{\rotatebox{90}{OSA$\to$KYO}} & Markov \cite{gambs2012next} & 0.401$\pm$0.063 & 0.219$\pm$0.063 & 0.372$\pm$0.030 & 0.264$\pm$0.030 & 0.473$\pm$0.012 & 0.206$\pm$0.003 & -- & --  & \multirow{7}{*}{\rotatebox{90}{\textcolor{red}{+8.71\%}}} \\
 & DITRAS \cite{pappalardo2018data} & 0.444$\pm$0.038 & 0.248$\pm$0.054 & 0.444$\pm$0.096 & 0.274$\pm$0.111 & \uline{0.401$\pm$0.024} & 0.175$\pm$0.002 & \uline{0.035$\pm$0.002} & \uline{0.280$\pm$0.003}  & \\
 & Gravity \cite{barbosa2018human} & 0.390$\pm$0.046 & \uline{0.177$\pm$0.040} & 0.456$\pm$0.086 & 0.377$\pm$0.125 & 0.420$\pm$0.013 & 0.151$\pm$0.001 & 0.116$\pm$0.001 & 0.352$\pm$0.002  & \\
 & MIRAGE \cite{deng2025revisiting} & \uline{0.386$\pm$0.049} & 0.178$\pm$0.048 & 0.462$\pm$0.113 & 0.308$\pm$0.164 & 0.660$\pm$0.275 & \uline{0.080$\pm$0.005} & 0.044$\pm$0.003 & 0.408$\pm$0.010  & \\
 & GTG \cite{wang2025gtg} & 0.431$\pm$0.015 & 0.232$\pm$0.015 & 0.432$\pm$0.016 & 0.380$\pm$0.037 & 0.462$\pm$0.052 & 0.158$\pm$0.003 & 0.175$\pm$0.002 & 0.290$\pm$0.005  & \\
 & COLA \cite{wang2024cola} & 0.505$\pm$0.089 & 0.308$\pm$0.099 & \textbf{0.355$\pm$0.097} & \textbf{0.247$\pm$0.157} & 0.597$\pm$0.024 & 0.260$\pm$0.004 & 0.105$\pm$0.004 & 0.324$\pm$0.004  & \\
 & \textbf{\method{}} & \textbf{0.374$\pm$0.036} & \textbf{0.172$\pm$0.029} & \uline{0.358$\pm$0.112} & \uline{0.247$\pm$0.164} & \textbf{0.336$\pm$0.017} & \textbf{0.063$\pm$0.004} & \textbf{0.026$\pm$0.001} & \textbf{0.276$\pm$0.003}  & \\
\cmidrule(l){1-11}
\multirow{7}{*}{\rotatebox{90}{OSA$\to$FUK}} & Markov \cite{gambs2012next} & \uline{0.358$\pm$0.093} & \uline{0.178$\pm$0.101} & \textbf{0.558$\pm$0.445} & \textbf{1.533$\pm$1.832} & \uline{0.252$\pm$0.005} & 0.182$\pm$0.004 & -- & --  & \multirow{7}{*}{\rotatebox{90}{\textcolor{blue}{-28.39\%}}} \\
 & DITRAS \cite{pappalardo2018data} & 0.504$\pm$0.059 & 0.343$\pm$0.042 & 0.748$\pm$0.263 & \uline{1.585$\pm$0.815} & \textbf{0.210$\pm$0.006} & 0.161$\pm$0.002 & 0.039$\pm$0.001 & \uline{0.311$\pm$0.008}  & \\
 & Gravity \cite{barbosa2018human} & 0.370$\pm$0.086 & 0.204$\pm$0.067 & \uline{0.615$\pm$0.404} & 1.747$\pm$1.983 & 0.273$\pm$0.017 & 0.184$\pm$0.004 & \uline{0.031$\pm$0.001} & 0.362$\pm$0.005  & \\
 & MIRAGE \cite{deng2025revisiting} & 0.424$\pm$0.078 & 0.228$\pm$0.084 & 0.834$\pm$0.474 & 2.220$\pm$2.287 & 0.832$\pm$0.322 & \uline{0.095$\pm$0.004} & 0.047$\pm$0.002 & 0.474$\pm$0.011  & \\
 & GTG \cite{wang2025gtg} & 0.434$\pm$0.119 & 0.396$\pm$0.241 & 1.043$\pm$0.130 & 4.053$\pm$0.950 & 0.768$\pm$0.025 & 0.134$\pm$0.002 & 0.141$\pm$0.002 & 0.331$\pm$0.011  & \\
 & COLA \cite{wang2024cola} & 0.564$\pm$0.067 & 0.430$\pm$0.080 & 0.734$\pm$0.387 & 2.204$\pm$1.755 & 0.270$\pm$0.006 & 0.221$\pm$0.005 & 0.100$\pm$0.003 & 0.337$\pm$0.004  & \\
 & \textbf{\method{}} & \textbf{0.339$\pm$0.081} & \textbf{0.169$\pm$0.088} & 0.794$\pm$0.455 & 2.643$\pm$2.076 & 0.541$\pm$0.012 & \textbf{0.080$\pm$0.004} & \textbf{0.027$\pm$0.002} & \textbf{0.292$\pm$0.003}  & \\
\cmidrule(l){1-11}
\multirow{7}{*}{\rotatebox{90}{NGO$\to$OSA}} & Markov \cite{gambs2012next} & 0.500$\pm$0.070 & 0.281$\pm$0.071 & 0.261$\pm$0.085 & \uline{0.116$\pm$0.072} & 0.465$\pm$0.009 & 0.215$\pm$0.002 & -- & --  & \multirow{7}{*}{\rotatebox{90}{\textcolor{red}{+15.23\%}}} \\
 & DITRAS \cite{pappalardo2018data} & 0.477$\pm$0.064 & 0.262$\pm$0.074 & 0.369$\pm$0.196 & 0.233$\pm$0.206 & 0.340$\pm$0.013 & 0.172$\pm$0.001 & 0.059$\pm$0.002 & \uline{0.215$\pm$0.003}  & \\
 & Gravity \cite{barbosa2018human} & 0.469$\pm$0.037 & 0.288$\pm$0.061 & 0.400$\pm$0.059 & 0.243$\pm$0.053 & \uline{0.338$\pm$0.011} & \uline{0.148$\pm$0.003} & \textbf{0.024$\pm$0.001} & 0.306$\pm$0.001  & \\
 & MIRAGE \cite{deng2025revisiting} & 0.379$\pm$0.046 & 0.175$\pm$0.035 & 0.418$\pm$0.080 & 0.301$\pm$0.075 & 0.589$\pm$0.198 & 0.160$\pm$0.001 & 0.034$\pm$0.002 & 0.271$\pm$0.005  & \\
 & GTG \cite{wang2025gtg} & 0.422$\pm$0.005 & 0.225$\pm$0.003 & \uline{0.257$\pm$0.023} & 0.137$\pm$0.011 & 0.662$\pm$0.014 & 0.176$\pm$0.001 & 0.187$\pm$0.001 & 0.256$\pm$0.004  & \\
 & COLA \cite{wang2024cola} & \uline{0.374$\pm$0.028} & \uline{0.167$\pm$0.014} & 0.269$\pm$0.026 & 0.135$\pm$0.007 & 0.484$\pm$0.015 & 0.257$\pm$0.002 & 0.131$\pm$0.001 & 0.267$\pm$0.001  & \\
 & \textbf{\method{}} & \textbf{0.372$\pm$0.037} & \textbf{0.161$\pm$0.039} & \textbf{0.201$\pm$0.052} & \textbf{0.072$\pm$0.055} & \textbf{0.285$\pm$0.007} & \textbf{0.071$\pm$0.001} & \uline{0.027$\pm$0.001} & \textbf{0.209$\pm$0.002}  & \\
\bottomrule
\end{tabular}
\end{table*}

\begin{table*}[t]
\centering
\small
\caption{Full results on \textbf{Turkish-cluster} transfer pairs. Conventions as in Table~\ref{tab:main}.}
\label{tab:app_turkey}
\begin{tabular}{ll cccc cccc c}
\toprule
& \multirow{2}{*}{Method} & \multicolumn{4}{c}{Downstream Utility} & \multicolumn{4}{c}{Distributional Fidelity}  &  \\
\cmidrule(lr){3-6}\cmidrule(lr){7-10}
& & Pred-MAPE\,$\downarrow$ & Pred-MSPE\,$\downarrow$ & Rec-MAPE\,$\downarrow$ & Rec-MSPE\,$\downarrow$ & OD-Sim\,$\downarrow$ & Distance\,$\downarrow$ & Speed\,$\downarrow$ & Time-Cat\,$\downarrow$  & \\
\midrule
\multirow{7}{*}{\rotatebox{90}{IST$\to$IZM}} & Markov \cite{gambs2012next} & \textbf{0.242$\pm$0.055} & \uline{0.081$\pm$0.041} & 0.316$\pm$0.177 & 0.216$\pm$0.147 & 0.502$\pm$0.010 & 0.162$\pm$0.002 & -- & --  & \multirow{7}{*}{\rotatebox{90}{\textcolor{blue}{-10.84\%}}} \\
 & DITRAS \cite{pappalardo2018data} & 0.319$\pm$0.066 & 0.138$\pm$0.032 & 0.375$\pm$0.046 & 0.238$\pm$0.043 & 0.408$\pm$0.012 & 0.127$\pm$0.002 & 0.062$\pm$0.003 & \uline{0.257$\pm$0.002}  & \\
 & Gravity \cite{barbosa2018human} & 0.291$\pm$0.088 & 0.123$\pm$0.067 & 0.242$\pm$0.166 & 0.123$\pm$0.129 & \textbf{0.334$\pm$0.003} & \uline{0.114$\pm$0.003} & \textbf{0.004$\pm$0.000} & 0.274$\pm$0.002  & \\
 & MIRAGE \cite{deng2025revisiting} & 0.434$\pm$0.029 & 0.229$\pm$0.028 & 0.322$\pm$0.133 & 0.224$\pm$0.122 & 0.976$\pm$0.033 & 0.146$\pm$0.003 & 0.063$\pm$0.003 & 0.380$\pm$0.009  & \\
 & GTG \cite{wang2025gtg} & 0.329$\pm$0.030 & 0.151$\pm$0.033 & 0.178$\pm$0.030 & 0.112$\pm$0.009 & 0.633$\pm$0.031 & 0.311$\pm$0.003 & 0.149$\pm$0.003 & 0.282$\pm$0.003  & \\
 & COLA \cite{wang2024cola} & 0.277$\pm$0.033 & 0.091$\pm$0.018 & \uline{0.151$\pm$0.076} & \uline{0.043$\pm$0.041} & 0.618$\pm$0.009 & 0.214$\pm$0.003 & 0.194$\pm$0.003 & 0.268$\pm$0.004  & \\
 & \textbf{\method{}} & \uline{0.246$\pm$0.030} & \textbf{0.074$\pm$0.021} & \textbf{0.105$\pm$0.043} & \textbf{0.023$\pm$0.019} & \uline{0.365$\pm$0.010} & \textbf{0.049$\pm$0.001} & \uline{0.013$\pm$0.001} & \textbf{0.240$\pm$0.002}  & \\
\cmidrule(l){1-11}
\multirow{7}{*}{\rotatebox{90}{IST$\to$ANK}} & Markov \cite{gambs2012next} & 0.352$\pm$0.043 & 0.176$\pm$0.032 & 0.362$\pm$0.110 & 0.226$\pm$0.106 & 0.623$\pm$0.011 & 0.161$\pm$0.002 & -- & --  & \multirow{7}{*}{\rotatebox{90}{\textcolor{red}{+7.60\%}}} \\
 & DITRAS \cite{pappalardo2018data} & 0.367$\pm$0.056 & 0.191$\pm$0.067 & 0.413$\pm$0.063 & 0.268$\pm$0.059 & 0.531$\pm$0.010 & 0.133$\pm$0.003 & 0.052$\pm$0.001 & 0.261$\pm$0.005  & \\
 & Gravity \cite{barbosa2018human} & 0.245$\pm$0.062 & 0.097$\pm$0.045 & \uline{0.235$\pm$0.106} & \uline{0.139$\pm$0.070} & \uline{0.472$\pm$0.015} & \uline{0.111$\pm$0.002} & \uline{0.015$\pm$0.004} & 0.263$\pm$0.005  & \\
 & MIRAGE \cite{deng2025revisiting} & 0.370$\pm$0.116 & 0.226$\pm$0.117 & 0.408$\pm$0.126 & 0.301$\pm$0.115 & 0.658$\pm$0.056 & 0.124$\pm$0.003 & 0.057$\pm$0.002 & 0.381$\pm$0.006  & \\
 & GTG \cite{wang2025gtg} & 0.308$\pm$0.016 & 0.119$\pm$0.006 & 0.287$\pm$0.007 & \textbf{0.125$\pm$0.009} & 0.712$\pm$0.022 & 0.297$\pm$0.005 & 0.146$\pm$0.002 & 0.283$\pm$0.007  & \\
 & COLA \cite{wang2024cola} & \uline{0.244$\pm$0.046} & \uline{0.080$\pm$0.037} & \textbf{0.224$\pm$0.054} & 0.146$\pm$0.039 & 0.692$\pm$0.013 & 0.196$\pm$0.001 & 0.185$\pm$0.002 & \uline{0.259$\pm$0.002}  & \\
 & \textbf{\method{}} & \textbf{0.235$\pm$0.029} & \textbf{0.079$\pm$0.018} & 0.251$\pm$0.065 & 0.140$\pm$0.061 & \textbf{0.408$\pm$0.018} & \textbf{0.052$\pm$0.002} & \textbf{0.014$\pm$0.001} & \textbf{0.242$\pm$0.002}  & \\
\cmidrule(l){1-11}
\multirow{7}{*}{\rotatebox{90}{IST$\to$BUR}} & Markov \cite{gambs2012next} & 0.384$\pm$0.111 & 0.143$\pm$0.107 & 0.499$\pm$0.141 & \uline{0.366$\pm$0.139} & 0.524$\pm$0.011 & 0.128$\pm$0.002 & -- & --  & \multirow{7}{*}{\rotatebox{90}{\textcolor{blue}{-16.10\%}}} \\
 & DITRAS \cite{pappalardo2018data} & 0.386$\pm$0.091 & \uline{0.141$\pm$0.098} & 0.564$\pm$0.236 & 0.710$\pm$0.442 & 0.420$\pm$0.016 & \uline{0.103$\pm$0.004} & 0.046$\pm$0.001 & 0.342$\pm$0.006  & \\
 & Gravity \cite{barbosa2018human} & 0.329$\pm$0.080 & 0.196$\pm$0.079 & \textbf{0.389$\pm$0.130} & \textbf{0.328$\pm$0.199} & \textbf{0.387$\pm$0.007} & 0.108$\pm$0.003 & \uline{0.019$\pm$0.000} & 0.352$\pm$0.006  & \\
 & MIRAGE \cite{deng2025revisiting} & \textbf{0.298$\pm$0.060} & \textbf{0.134$\pm$0.044} & 0.626$\pm$0.120 & 0.827$\pm$0.273 & 0.987$\pm$0.043 & 0.137$\pm$0.002 & 0.064$\pm$0.004 & 0.509$\pm$0.004  & \\
 & GTG \cite{wang2025gtg} & \uline{0.305$\pm$0.025} & 0.157$\pm$0.016 & 0.542$\pm$0.083 & 0.713$\pm$0.224 & 0.776$\pm$0.025 & 0.291$\pm$0.008 & 0.146$\pm$0.003 & 0.369$\pm$0.005  & \\
 & COLA \cite{wang2024cola} & 0.400$\pm$0.149 & 0.224$\pm$0.146 & \uline{0.432$\pm$0.224} & 0.526$\pm$0.408 & 0.639$\pm$0.013 & 0.179$\pm$0.002 & 0.154$\pm$0.004 & \uline{0.336$\pm$0.005}  & \\
 & \textbf{\method{}} & 0.321$\pm$0.035 & 0.148$\pm$0.038 & 0.527$\pm$0.238 & 0.749$\pm$0.421 & \uline{0.406$\pm$0.011} & \textbf{0.051$\pm$0.004} & \textbf{0.018$\pm$0.002} & \textbf{0.328$\pm$0.004}  & \\
\cmidrule(l){1-11}
\multirow{7}{*}{\rotatebox{90}{IZM$\to$IST}} & Markov \cite{gambs2012next} & 0.226$\pm$0.089 & 0.097$\pm$0.060 & 0.205$\pm$0.101 & 0.097$\pm$0.084 & 0.628$\pm$0.014 & 0.158$\pm$0.002 & -- & --  & \multirow{7}{*}{\rotatebox{90}{\textcolor{red}{+10.00\%}}} \\
 & DITRAS \cite{pappalardo2018data} & 0.259$\pm$0.051 & 0.099$\pm$0.033 & 0.258$\pm$0.122 & 0.131$\pm$0.109 & 0.557$\pm$0.019 & 0.133$\pm$0.001 & 0.046$\pm$0.001 & \uline{0.225$\pm$0.001}  & \\
 & Gravity \cite{barbosa2018human} & 0.310$\pm$0.077 & 0.134$\pm$0.044 & 0.353$\pm$0.200 & 0.220$\pm$0.201 & \uline{0.514$\pm$0.008} & 0.091$\pm$0.002 & \uline{0.020$\pm$0.000} & 0.238$\pm$0.002  & \\
 & MIRAGE \cite{deng2025revisiting} & 0.335$\pm$0.094 & 0.205$\pm$0.100 & 0.411$\pm$0.101 & 0.341$\pm$0.113 & 0.740$\pm$0.144 & \uline{0.071$\pm$0.001} & 0.027$\pm$0.001 & 0.320$\pm$0.008  & \\
 & GTG \cite{wang2025gtg} & \uline{0.217$\pm$0.011} & \uline{0.079$\pm$0.004} & 0.135$\pm$0.036 & 0.034$\pm$0.015 & 0.681$\pm$0.018 & 0.316$\pm$0.002 & 0.163$\pm$0.002 & 0.234$\pm$0.004  & \\
 & COLA \cite{wang2024cola} & 0.255$\pm$0.060 & 0.089$\pm$0.038 & \textbf{0.118$\pm$0.047} & \textbf{0.027$\pm$0.020} & 0.838$\pm$0.006 & 0.266$\pm$0.002 & 0.207$\pm$0.002 & 0.228$\pm$0.002  & \\
 & \textbf{\method{}} & \textbf{0.196$\pm$0.032} & \textbf{0.058$\pm$0.018} & \uline{0.134$\pm$0.062} & \uline{0.031$\pm$0.022} & \textbf{0.451$\pm$0.014} & \textbf{0.051$\pm$0.001} & \textbf{0.015$\pm$0.001} & \textbf{0.210$\pm$0.001}  & \\
\cmidrule(l){1-11}
\multirow{7}{*}{\rotatebox{90}{ANK+IZM$\to$IST}} & Markov \cite{gambs2012next} & 0.292$\pm$0.087 & 0.137$\pm$0.057 & 0.315$\pm$0.120 & 0.186$\pm$0.084 & 0.658$\pm$0.015 & 0.166$\pm$0.001 & -- & --  & \multirow{7}{*}{\rotatebox{90}{\textcolor{red}{+16.97\%}}} \\
 & DITRAS \cite{pappalardo2018data} & 0.338$\pm$0.086 & 0.164$\pm$0.065 & 0.278$\pm$0.182 & 0.156$\pm$0.156 & 0.565$\pm$0.011 & 0.135$\pm$0.001 & 0.048$\pm$0.001 & \textbf{0.213$\pm$0.002}  & \\
 & Gravity \cite{barbosa2018human} & 0.346$\pm$0.058 & 0.157$\pm$0.053 & 0.207$\pm$0.112 & 0.103$\pm$0.089 & \uline{0.507$\pm$0.008} & \uline{0.089$\pm$0.002} & \uline{0.026$\pm$0.000} & 0.238$\pm$0.003  & \\
 & MIRAGE \cite{deng2025revisiting} & 0.277$\pm$0.113 & 0.116$\pm$0.057 & 0.290$\pm$0.110 & 0.201$\pm$0.105 & 0.784$\pm$0.094 & 0.120$\pm$0.001 & 0.036$\pm$0.001 & 0.299$\pm$0.002  & \\
 & GTG \cite{wang2025gtg} & \uline{0.207$\pm$0.020} & \uline{0.077$\pm$0.005} & 0.200$\pm$0.044 & \uline{0.050$\pm$0.016} & 0.674$\pm$0.014 & 0.317$\pm$0.003 & 0.163$\pm$0.002 & 0.237$\pm$0.008  & \\
 & COLA \cite{wang2024cola} & 0.240$\pm$0.035 & 0.082$\pm$0.025 & \uline{0.183$\pm$0.108} & 0.072$\pm$0.080 & 0.830$\pm$0.017 & 0.266$\pm$0.003 & 0.207$\pm$0.002 & 0.226$\pm$0.003  & \\
 & \textbf{\method{}} & \textbf{0.205$\pm$0.018} & \textbf{0.064$\pm$0.010} & \textbf{0.172$\pm$0.063} & \textbf{0.049$\pm$0.032} & \textbf{0.459$\pm$0.010} & \textbf{0.045$\pm$0.001} & \textbf{0.012$\pm$0.001} & \uline{0.219$\pm$0.002}  & \\
\bottomrule
\end{tabular}
\end{table*}

\subsection{Parameter Analysis}
The supplemented sensitivity analysis further provides the impacts of three coefficients in \textbf{\method{}} regulate complementary components of the framework. Specifically, $\lambda_{\mathrm{cat}}$ weights the category prediction loss for cross-city semantic alignment, $\lambda_{\mathrm{link}}$ controls the link reconstruction loss that preserves the structural relationships among POIs, and $\lambda_{\mathrm{gate}}$ balances the return gate and revisit prediction objectives in the Pointer Network. Across both transfer settings, competitive performance is maintained over a broad range of parameter combinations rather than being confined to a single sharp optimum, indicating that \textbf{\method{}} is not highly sensitive to fine-grained tuning. Noticeable degradation mainly occurs when either $\lambda_{\mathrm{cat}}$ or $\lambda_{\mathrm{link}}$ becomes overly dominant, as excessive emphasis on category semantics or graph structure can weaken the balance required for learning transferable POI representations. In comparison, varying $\lambda_{\mathrm{gate}}$ causes only moderate changes in the error landscape, suggesting that the return-and-exploration mechanism remains stable under different relative weights. The similar response patterns on {ANK+BUR+IZM$\to$IST} and {KYO+NGO$\to$OSA} further confirm that the selected hyper-parameters generalize across different city groups rather than being tailored to a particular transfer setting.

\begin{figure*}[htbp]
\centering
\vspace{0.1in}
\subfigtopskip=-1pt 
\subfigcapskip=-5pt 
\subfigure[$\lambda_{gate}=0.25$]{
	\label{l_1} 
	\includegraphics[width=0.23\linewidth]{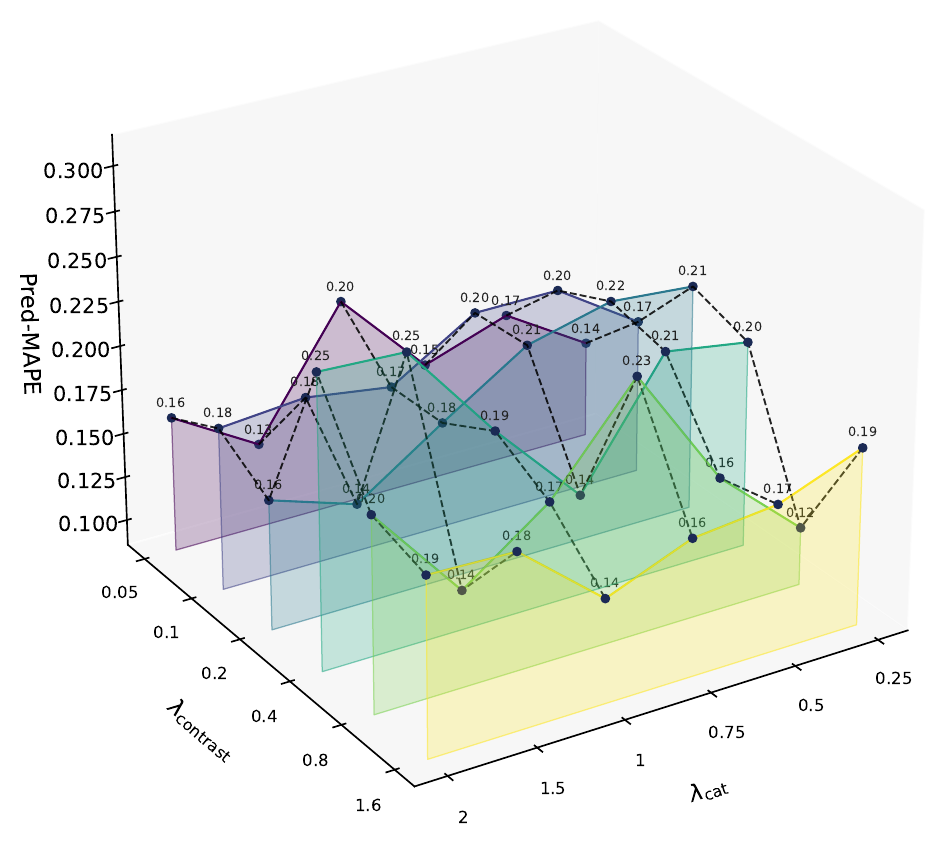}}
\hspace{0in}
\subfigure[$\lambda_{gate}=1$]{
	\label{l_2}
	\includegraphics[width=0.23\linewidth]{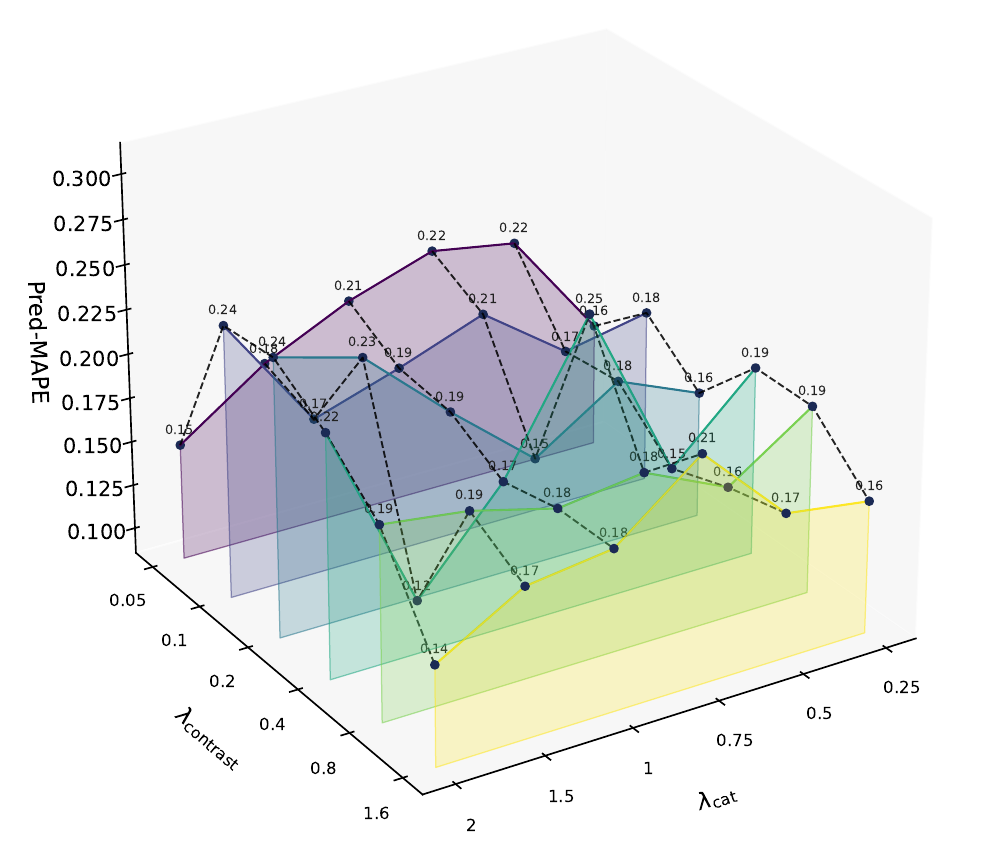}}
\hspace{0in}
\subfigure[$\lambda_{gate}=4$]{
	\label{l_3}
	\includegraphics[width=0.23\linewidth]{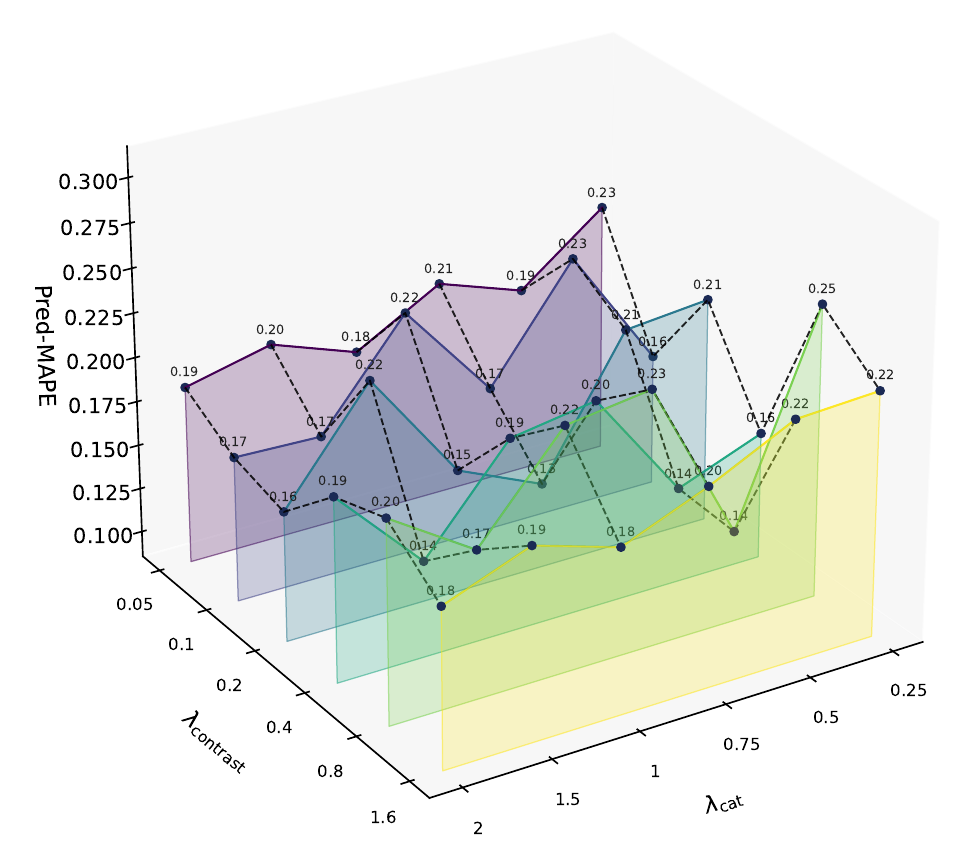}}
\hspace{0in}
\subfigure[$\lambda_{gate}=8$]{
	\label{l_3}
	\includegraphics[width=0.23\linewidth]{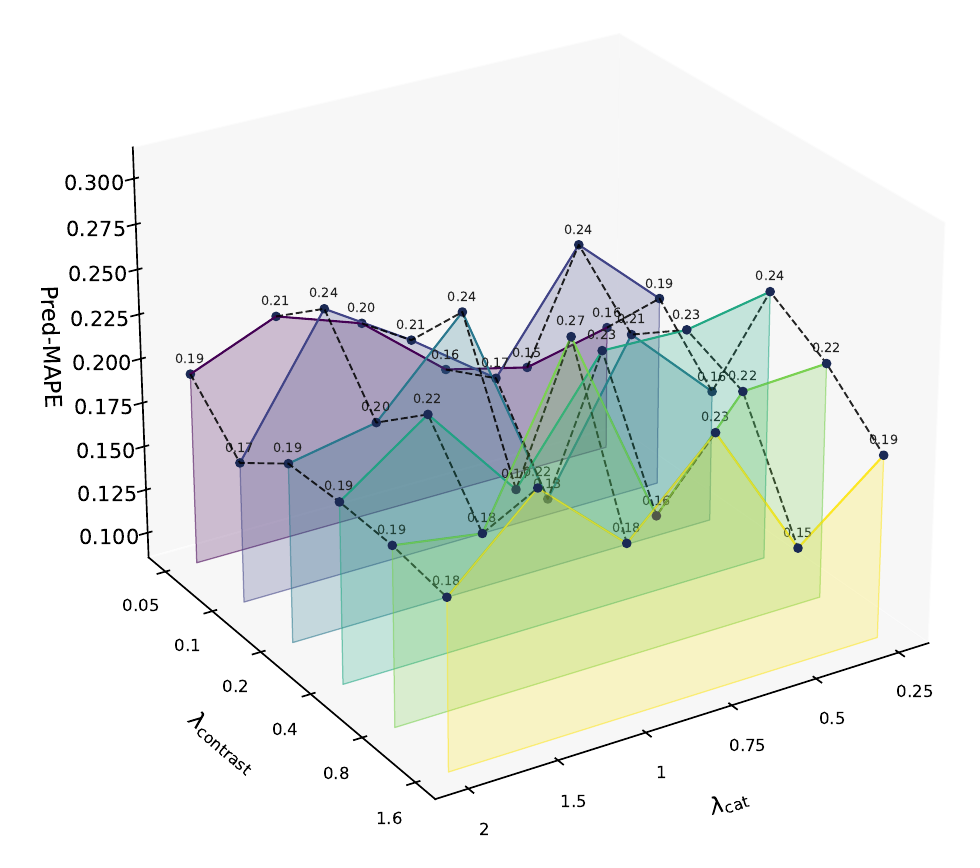}}
\caption{Hyper-parameter sensitivity on \{ANK+BUR+IZM$\to$IST\}.}
\label{parameters1} 
\end{figure*}

\begin{figure*}[htbp]
\centering
\vspace{0.1in}
\subfigtopskip=-1pt 
\subfigcapskip=-5pt 
\subfigure[$\lambda_{gate}=0.25$]{
	\label{l_1} 
	\includegraphics[width=0.23\linewidth]{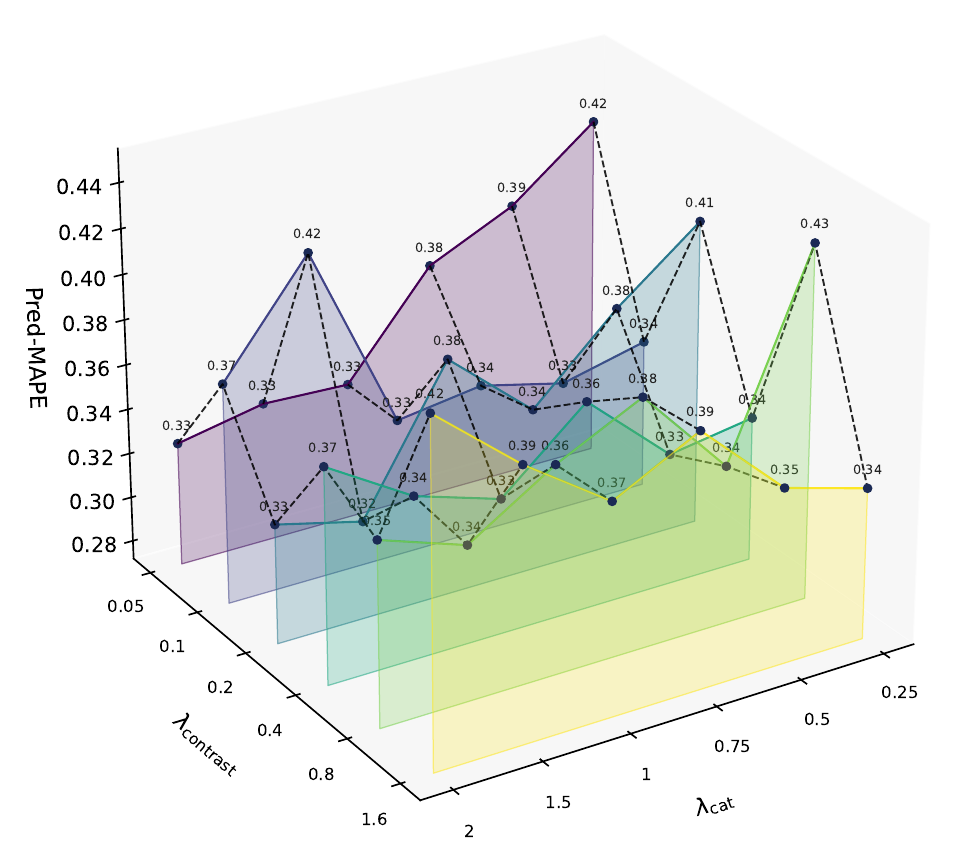}}
\hspace{0in}
\subfigure[$\lambda_{gate}=1$]{
	\label{l_2}
	\includegraphics[width=0.23\linewidth]{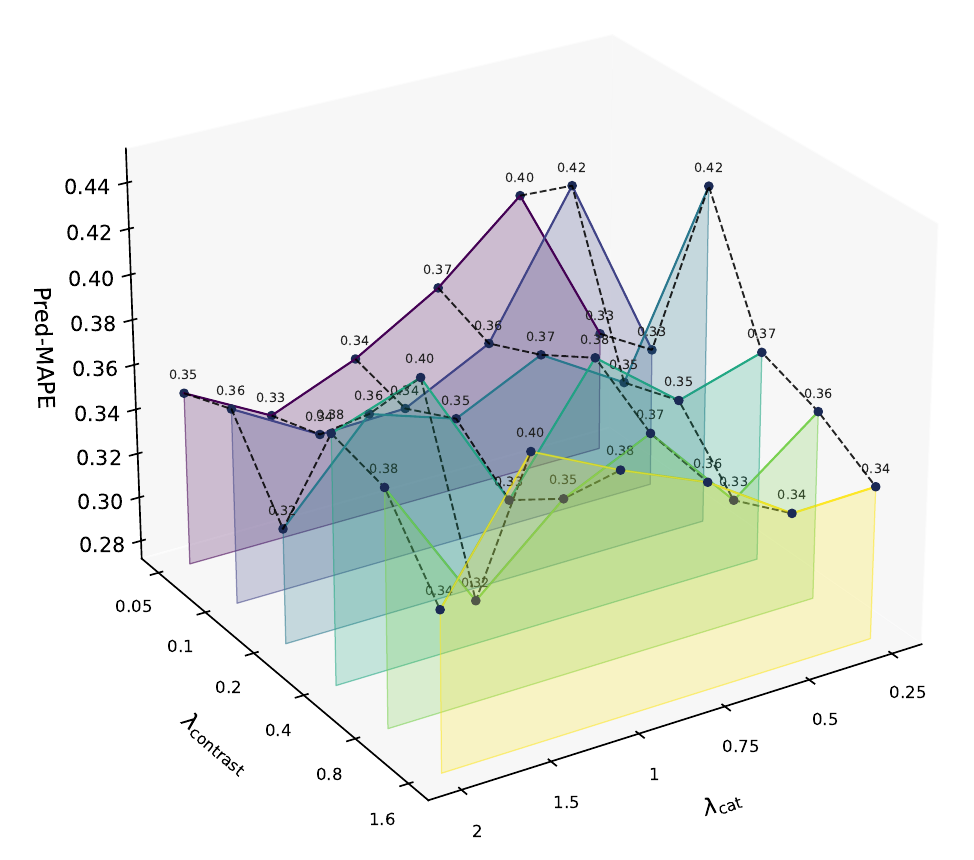}}
\hspace{0in}
\subfigure[$\lambda_{gate}=4$]{
	\label{l_3}
	\includegraphics[width=0.23\linewidth]{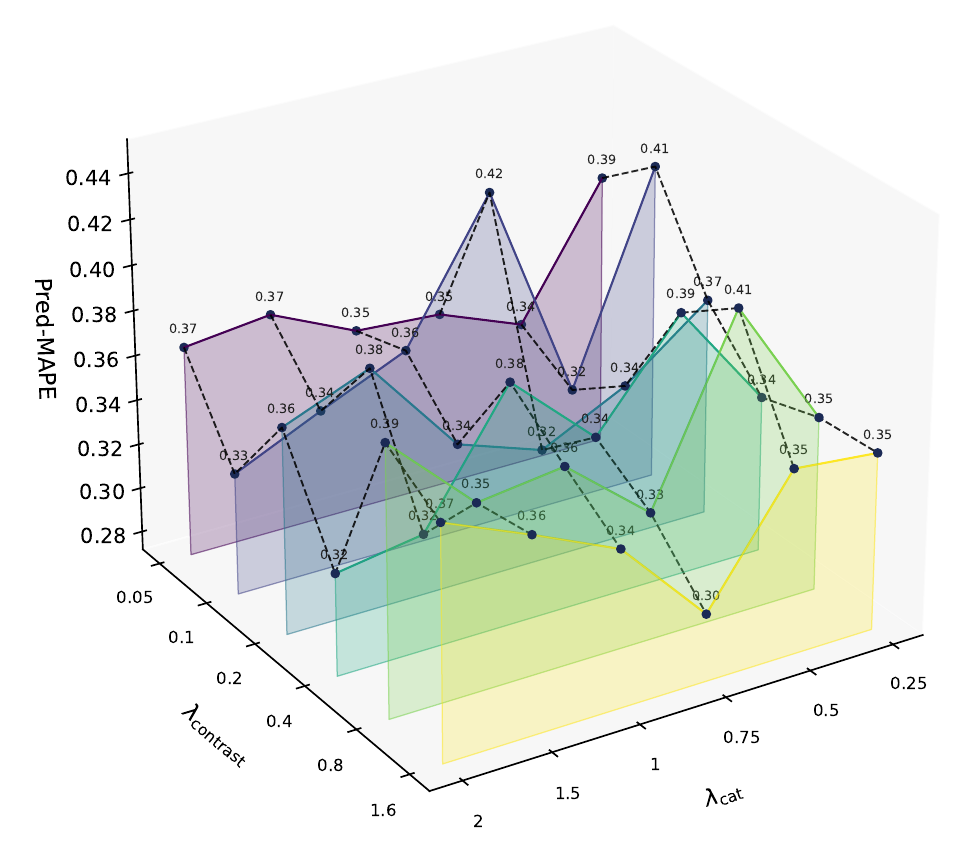}}
\hspace{0in}
\subfigure[$\lambda_{gate}=8$]{
	\label{l_3}
	\includegraphics[width=0.23\linewidth]{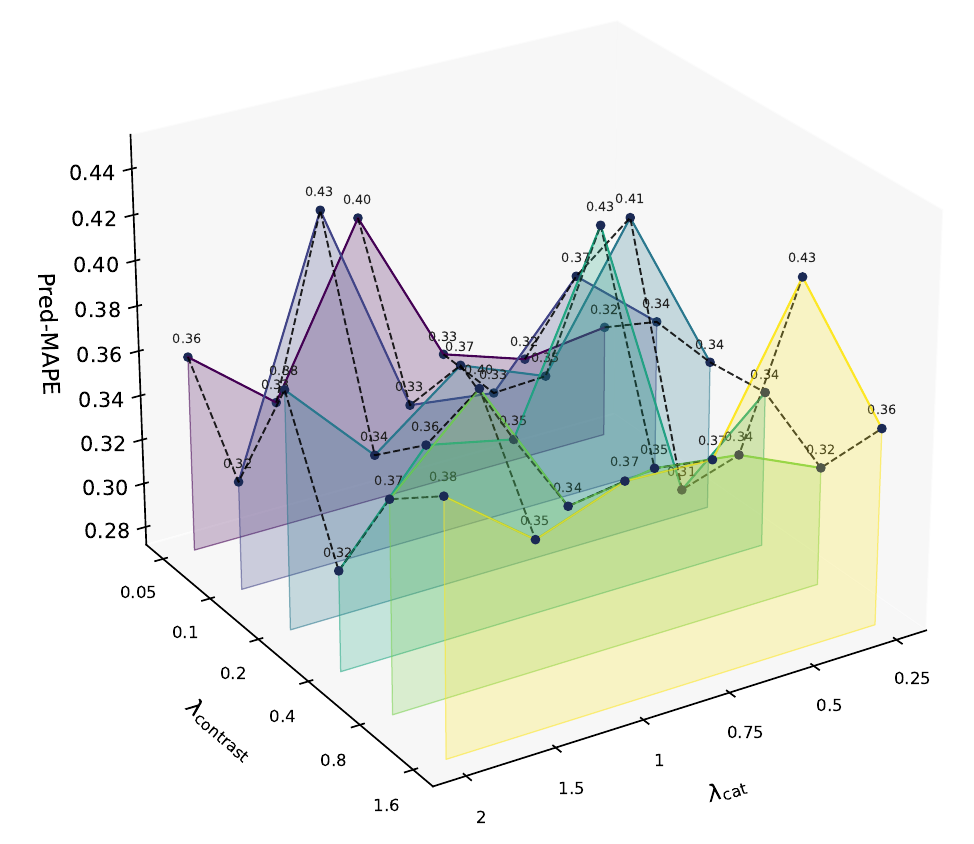}}
\caption{Hyper-parameter sensitivity on \{KYO+NGO$\to$OSA\}.}
\label{parameters1} 
\end{figure*}


\end{document}